\documentclass[lettersize,journal]{IEEEtran}
\usepackage{amsmath,amsfonts}
\usepackage{algorithmic}
\usepackage{algorithm}
\usepackage{amsthm}
\usepackage{array}
\usepackage[caption=false,font=normalsize,labelfont=sf,textfont=sf]{subfig}
\usepackage{textcomp}
\usepackage{stfloats}
\usepackage{xurl} 
\usepackage{hyperref}
\usepackage{verbatim}
\usepackage{graphicx}
\usepackage{cite}
\usepackage{booktabs}
\usepackage{multirow}
\usepackage{threeparttable}
\usepackage[normalem]{ulem}
\providecommand{\tabularnewline}{\\}

\newtheorem{lemma}{Lemma}

\begin{document}

\title{TSNN: A Non-parametric and Interpretable Framework for Traffic Time Series Forecasting}


\author{Bowen Liu, Haijian Lai,  
Chan-Tong Lam,
Junhao Dong,
Benjamin Ng, 
Wei Ke,
and Sio-Kei Im
\thanks{This article has been accepted for publication in IEEE Transactions on Knowledge and Data Engineering.}
\thanks{This work was supported by the Science and Technology Development Fund of Macao (FDCT) \#0033/2023/RIA1.}
\thanks{Bowen Liu; Haijian Lai; Chan-tong Lam, Benjamin K. Ng and Wei Ke are with the Faculty of Applied Sciences, Macao Polytechnic University, Macao SAR, China. (e-mail: \{bowen.liu, haijian.lai, ctlam, bng, wke\}@mpu.edu.mo).}
\thanks{Junhao Dong is with Amazon Web Services, Inc. This work is not related to the role at Amazon. (e-mail: derekdong2022@gmail.com).}
\thanks{Sio-kei Im is with Macao Polytechnic University, Macao SAR, China. (e-mail: marcusim@mpu.edu.mo).}}

\markboth{Journal of \LaTeX\ Class Files,~Vol.~14, No.~8, August~2021}%
{Shell \MakeLowercase{\textit{et al.}}: A Sample Article Using IEEEtran.cls for IEEE Journals}

\IEEEpubid{0000--0000/00\$00.00~\copyright~2021 IEEE}

\maketitle

\begin{abstract}

Although many complex models were proposed to analyze time series data, some studies have demonstrated remarkable performance with simpler structures. A recent study proposed a non-parametric framework for 3D point cloud classification, which has the potential to be adapted for time series forecasting and enable interpretability. Inspired by the previous works, we present TSNN, a non-parametric and interpretable framework for traffic time series forecasting. 
TSNN consists of multiple layers that decouple the time series by matching the entries in a memory bank, where the memory bank is constructed using a similar matching process within the training set. It leverages the periodicity in traffic data to enhance forecasting accuracy while maintaining a simple model architecture. 
The proposed model operates without trainable parameters, preserving its inherent interpretability.
In the experiments, TSNN achieves competitive performance compared to the typical deep learning models in four real-world traffic flow datasets. We also visualize the decoupling process to show the effectiveness of the components. Finally, we demonstrate the interpretability of the model and illustrate the contribution of each time step within the memory bank. Our code is available at \url{https://github.com/pzzzzzm/TSNN_release}.

\end{abstract}


\section{Introduction}

Time series forecasting is the task of predicting future time series data by fitting the historical data \cite{lim2021time}, which plays a pivotal role in various fields, including but not limited to transportation \cite{lana2018road}, electricity \cite{zhou2021informer}, and climate modeling \cite{mudelsee2019trend}. Traffic prediction is one of the most important applications of time series forecasting, enhancing congestion control, urban planning, and transportation efficiency \cite{liu2024largest, medina2022urban}. 

A key characteristic of time series forecasting for traffic data is the temporal dependency and periodicity.
Strong periodicity often influences traffic data, which leads to daily or weekly patterns. Thus, efficiently leveraging periodicity becomes a challenge in traffic forecasting. Moreover, handling the diversity and complexity of time series data is challenging to propose general and efficient forecasting approaches \cite{kim2025comprehensive}. For the traffic data with multiple sensors, the spatial dependency also benefits the prediction accuracy.

Early studies for traffic forecasting are mostly based on statistic-based methods, with the most prominent being the auto-regressive integrated moving
average (ARIMA) method \cite{moorthy1988short, lee1999application} and Kalman filters \cite{guo2014adaptive}. Then, machine learning approaches were introduced into this field, such as Support Vector Regression (SVR) \cite{wu2004travel}, to achieve more accurate predictions for complex time series. After the advances of deep learning, many studies utilized spatial dependency to complement the limited information of time series \cite{li2018diffusion, wu2019graph, yu2017spatio}, which adopts graph convolution on multivariate time series data. Besides, recent studies used transformer-based models to analyze the traffic time series \cite{cai2020traffic, luo2024lsttn, chen2022bidirectional}.

Although the models have become increasingly complex to enhance accuracy with the cost of inefficiency, a model utilizing spatial and temporal identity information (STID) \cite{shao2022spatial} with only multilayer perceptrons (MLP) achieves competitive performance in both accuracy and efficiency. STID explicitly embeds spatial and temporal information of the input time series data and concatenates them with the raw input for the latter process. In the time series in traffic forecasting, the data is usually indistinguishable in time dimension because it only provides limited information. However, when the time label is considered, the strong periodicity of the traffic time series is useful and crucial for the prediction result \cite{yao2019revisiting}. Based on this, we may infer that periodicity can provide basic patterns for traffic forecasting if it is properly utilized, while the structure of models provides the capability of analyzing the rest of the information in a time series. Therefore, achieving an adequate performance with a simple model is possible.

Recently, a study in point cloud classification proposed a non-parametric framework, named Point-NN \cite{zhang2023parameter}, based on the structure of a deep learning model PointNet++ \cite{qi2017pointnet++}. Point-NN builds a memory bank by transforming the training data and classifies the test data into the class with the highest matching score. With no trainable parameter included, it surpasses the baseline neural network in the accuracy of 3D point cloud classification. Meanwhile, the non-parametric structure has the potential to provide inherent interpretations of the prediction, enhancing the reliability of time series forecasting in crucial decision-making applications. Motivated by the success of non-parametric frameworks and the periodicity in time series, we raise the question: can we design a non-parametric framework for time series forecasting that leverages periodic patterns?

\IEEEpubidadjcol

We propose TSNN, a \textbf{N}on-parametric \textbf{N}etwork for traffic \textbf{T}ime \textbf{S}eries forecasting. It consists of multiple layers that decouple the time series by calculating the similarity score with the entries in a memory bank. The memory bank is constructed using a similar matching process within the training set, where the training data match each other. The prediction is aggregated from the entries in the memory bank based on the similarities between the entries and input data, as well as the similarities between the entries and residual signals. The first layer leverages the periodicity in traffic data to enhance forecasting accuracy while maintaining a straightforward and efficient architecture. 
The traceability of memory bank entries to the data source, along with the absence of trainable parameters, enables the inherent interpretability of TSNN.
We propose an interpretation method that extracts and summarizes the similarity scores. The experiments on TSNN verify the effectiveness of the proposed structure and components, which can provide insights into the area of time series forecasting.

Our contributions are summarized as follows:
\begin{itemize}
    \item We propose TSNN, a non-parametric and interpretable framework for traffic time series forecasting, based on a memory bank and layerized structure;
    \item We propose time-wise decoupling and mean decoupling components to enhance forecasting accuracy;
    \item We visualize the interpretability of TSNN and conduct experiments to verify the effectiveness of the proposed framework and components. 
\end{itemize}

The remaining of the paper is organized as follows. Section \ref{sec:related} introduces the related works about traffic time series forecasting and non-parametric models. Section \ref{sec:prel} covers the preliminary knowledge. Section \ref{sec:method} presents the structure of the proposed framework. Section \ref{sec:exper} demonstrates the experiments, including the performance, ablation study, visualization, etc. Section \ref{sec:limit} discusses the limitations of the framework. Finally, Section \ref{sec:conclu} concludes the paper.

\section{Related Works} \label{sec:related}

\subsection{Time Series Forecasting on Traffic}

Early studies on time series forecasting were mostly based on statistical methods. ARIMA is a widely used model for time series prediction \cite{moorthy1988short, lee1999application}.  Some varieties of ARIMA include seasonal components to improve the ability to handle the periodicity, such as SARIMA and SARIMAX \cite{lippi2013short, vagropoulos2016comparison}. Vector Autoregressive (VAR) adapts the autoregressive to multivariate time series by calculating the weighted sum of vectors in the past time steps \cite{stock2001vector}. The model-driven structures of statistical models enable better interpretability.

As data-driven models, neural networks are more capable of fitting complex functions even if the functions are implicit. Due to the natural capability of process sequence data, the neural networks based on Recurrent Neural Networks (RNNs) are widely used in time series forecasting \cite{sagheer2019time, cao2019financial, kumar2018energy, zhao2017lstm, ma2020multi}. However, RNN-based models struggle with the iterative prediction process which accumulates errors. Later studies introduced transformer \cite{vaswani2017attention} and attention mechanisms to the time series models \cite{wang2024deep}. Zhou et al. proposed Informer \cite{zhou2021informer} that adapts transformer for long-term time series forecasting by reducing the complexity of the self-attention mechanism, highlighting the main attention information. Informer successfully verifies the effectiveness of transformers in time series forecasting. Another concurrent work proposed Pyraformer \cite{liu2022pyraformer} which introduces a pyramidal attention module to reduce the complexity of retrieving temporal dependencies of different ranges. In addition, some following works leveraged patch-wise attention to compute the attention scores between sub-series, which enables a more intuitive and efficient approach \cite{wu2021autoformer, nie2022time, zhang2023crossformer}.

Due to the nature of spatial dependency in traffic networks, spatial-temporal graph neural networks (STGCNs) are proposed to handle spatial dependency by graphs \cite{shao2024exploring, yu2017spatio}. Li et al. proposed DCRNN \cite{li2017diffusion} for traffic forecasting, modeling the signal diffusion in the traffic network by random walks. STGCN \cite{yu2017spatio} applies convolution in both spatial and temporal dependency. GMAN \cite{zheng2020gman} designs spatio-temporal attention blocks that apply the attention mechanism to traffic networks. Graph WaveNet \cite{wu2019graph} proposes a self-learned hidden graph in the graph convolution module and adapts dilated convolution to aggregate temporal dependency. The following studies use a dynamic graph to complement dynamic dependency into the predefined graph \cite{li2021dynamic, shao2022decoupled}.

Besides the above models, some studies leveraged multilayer perceptrons (MLPs) and achieved impressive performance. N-beats \cite{oreshkin2019n} provides interpretable outputs by utilizing MLPs and seasonality functions. STID \cite{shao2022spatial} achieves competitive performance with only MLPs and spatial and temporal embeddings. 
Even in the field of long-term forecasting, many models based on linear layers, such as DLinear \cite{zeng2023transformers}, RLinear \cite{li2023revisiting}, challenge the necessity of transformer-based models.
More recent studies adapt the Mixture-of-Experts techniques to the time series models \cite{ni2024mixture, lee2024testam, wang2024time}.

\subsection{Non-parametric Models}
Non-parametric models generally refer to the models that do not make strong assumptions or priors to the mapping function form, avoiding finding the best parameters with trial-and-error procedures \cite{carmon2022making,goyal2017nonparametric}. Non-parametric regression has been developed in many studies \cite{aneiros2011functional}. For instance, the Nadaraya–Watson estimator \cite{nadaraya1964estimating, watson1964smooth} was proposed to estimate the local weighted average using kernel regression.
Another study used the Gaussian Process \cite{williams1995gaussian} to optimize the non-parametric Bayesian analysis of neural networks. The non-parametric kernel regression has been used to analyze the spatio-temporal shape to address the issue of the limited expressive power of flat Euclidean spaces for the biological variability \cite{davis2010population}. 
Similarly, a non-parametric predictive methodology for manifold-valued objects is proposed to model complex data that do not comply with Euclidean geometry, such as images and shapes \cite{TSAGKRASOULIS20186}. Rangapuram et al. \cite{rangapuram2023deep} proposed a non-parametric time series forecaster (NPTS) that predicts the future time series by sampling the input series. 


In recent machine learning studies, the definition and scope of non-parametric models have been extended to include models that do not involve trainable parameters. 
In the field of computer vision, Wu et al. \cite{wu2018unsupervised} proposed a non-parametric classifier based on the memory bank which consists of instance-level discriminative features. 
Afterward, MoCo \cite{he2020momentum} improved this structure by dynamically updating the memory bank. 
Beyond non-parametric classifiers, some works presented the non-parametric models by replacing trainable components with non-parametric components. For instance, Point-NN \cite{zhang2023parameter}, Seg-NN \cite{zhu2024no}, and Point-FCW \cite{lai2025point} showed that non-parametric backbones can achieve competitive performance in 3D point cloud classification and segmentation.




\section{Preliminary} \label{sec:prel}

\subsection{Traffic Forecasting}

Traffic forecasting is the task of predicting future traffic signals based on given historical traffic signals. The traffic signals are collected from the sensors in road networks. For the non-parametric framework based on a memory bank, the training data or its transformation can be considered as the input directly. Let $x_i \in \mathbb{R}$ denote the traffic signal at time step $i$, the objective is to map $T$ historical traffic signals to $T^\prime$ future traffic signals by a function $f$,
\begin{equation}
    \Big[x_{i-T+1}, \dots,  x_i; \{x_{i^\prime}\}_{i^\prime \in \mathcal{T}}\Big] \xrightarrow{f} \Big[y_{i+1}, \dots, y_{i+T^\prime}\Big],
\end{equation}
where $\mathcal{T}$ is the set of time steps of training data and $y$ denotes the prediction. 

\subsection{Memory Bank}
The concept of a memory bank is widely utilized in non-parametric and contrastive learning to maintain a dictionary of feature representations \cite{wu2018unsupervised, he2020momentum, zhang2023parameter, lai2025point}. 
Let $X_i$ denote the sample indexed by $i \in \{1, 2, \dots, n\}$ in the dataset with $n$ samples, and $B$ denote the memory bank. A task-specific model $\phi$ transforms the sample into the feature representation, denoted by $\phi(X_i)$. Each entry $b_i \in B$ corresponds to the feature representation of the $i$-th sample. During the construction of the memory bank, the feature representation $\phi(X_i)$ and the label $Y_i$ of the sample $X_i$ are assigned to the corresponding entry, which is represented by,
\begin{equation}
    b_i = \langle \phi(X_i), Y_i \rangle.
\end{equation}
In such frameworks, inference is typically performed based on the feature similarity between the input and stored representations in entries.
Memory banks allow the model to perform operations directly on instance-level features, enabling instance-wise reasoning without aggregating them into weights. 


\section{Methodology} \label{sec:method}

In this section, we present TSNN, a non-parametric framework for traffic forecasting. As shown in Fig. \ref{fig:arch}, it consists of two stages: memory bank construction and prediction. The prediction and residual signals are calculated based on the similarity scores between the input and entries in the memory bank. The structure inside each stage is stacked with layers that decouple the time series. The details of the framework's components are provided in the following sections. Section \ref{sec:model_ts} and \ref{sec:np_reg} introduce two basic concepts of TSNN. Section \ref{sec:dec_ts} depicts the detailed structure of the proposed framework.

\begin{figure*}[t]
    \centering
    \includegraphics[width=1\linewidth,trim={8 0 0 0}, clip]{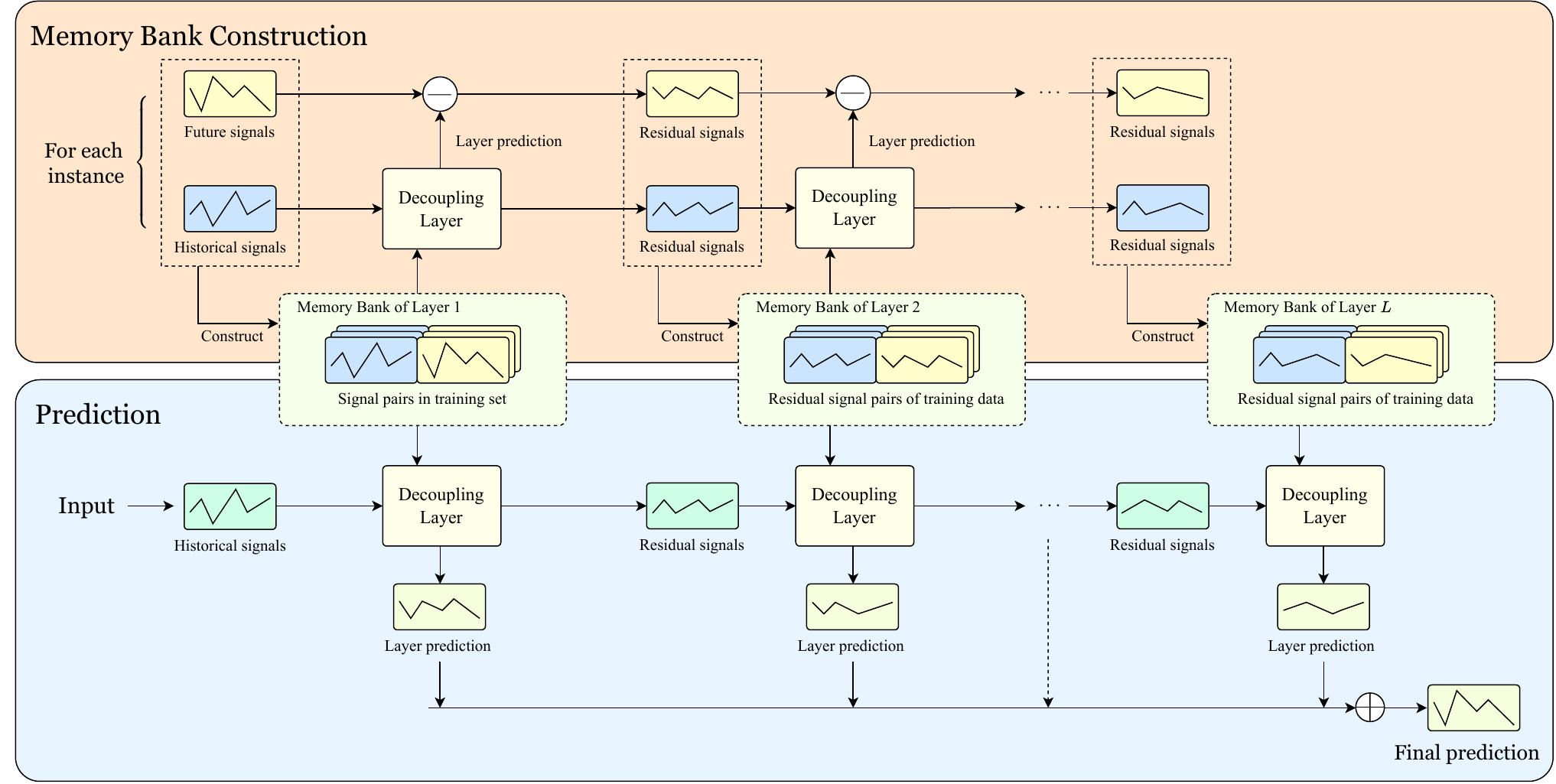}
    \caption{
    The architecture of TSNN. The framework operates in two stages: memory bank construction (top) and prediction (bottom). In the construction stage, each layer iteratively decouples the training signals into residuals, which are stored in the corresponding memory bank. In the prediction stage, the model retrieves information from these banks to generate layer predictions and pass the residuals to the next layer. The final output is the summation of all layer predictions.
    }
    \label{fig:arch}
\end{figure*}

\subsection{Modelling Traffic Signal} \label{sec:model_ts}

Many studies on traffic forecasting utilize the periodicity of the traffic network \cite{yao2019revisiting, shao2022spatial, shao2022decoupled}. Figure \ref{fig:ha_slice} illustrates a slice of data from sensor 120 in PEMS08 dataset with its daily historical average values. Note that the duration of 288 time steps equals one day as the interval between two time steps is 5 minutes. We can observe that the value at the same time on different days is generally similar, despite some perturbation. It reveals patterns on a curve that illustrate the fundamental trends in the data. Therefore, we can model a fundamental trend at a time step as an implicit pattern, which forms the traffic signal with the perturbation caused by external factors. 
Define the periodic time step as a recurring time point within the time series that follows a specific cycle. For example, the time point of 12:00 is a periodic time step that corresponds to multiple time steps in the whole time series.
Given a traffic signal sequence $X_i$ starting from the time step $i$ and denote the periodic time step of $i$ as $p(i)$ where $p(i) = i \mod t$ and $t$ is the number of time step in a period, we model it by,
\begin{equation} \label{equ:xi_def}
    X_i = X^{imp}_{p(i)} + X^{dif}_i,
\end{equation}
where $X^{imp}_{p(i)}$ represents the implicit pattern at the periodic time step $p(i)$ and $X^{dif}_i$ represents the perturbation at this time step.

\begin{figure}[htbp]
    \centering
    \includegraphics[width=1\linewidth,trim={30 0 60 50}, clip]{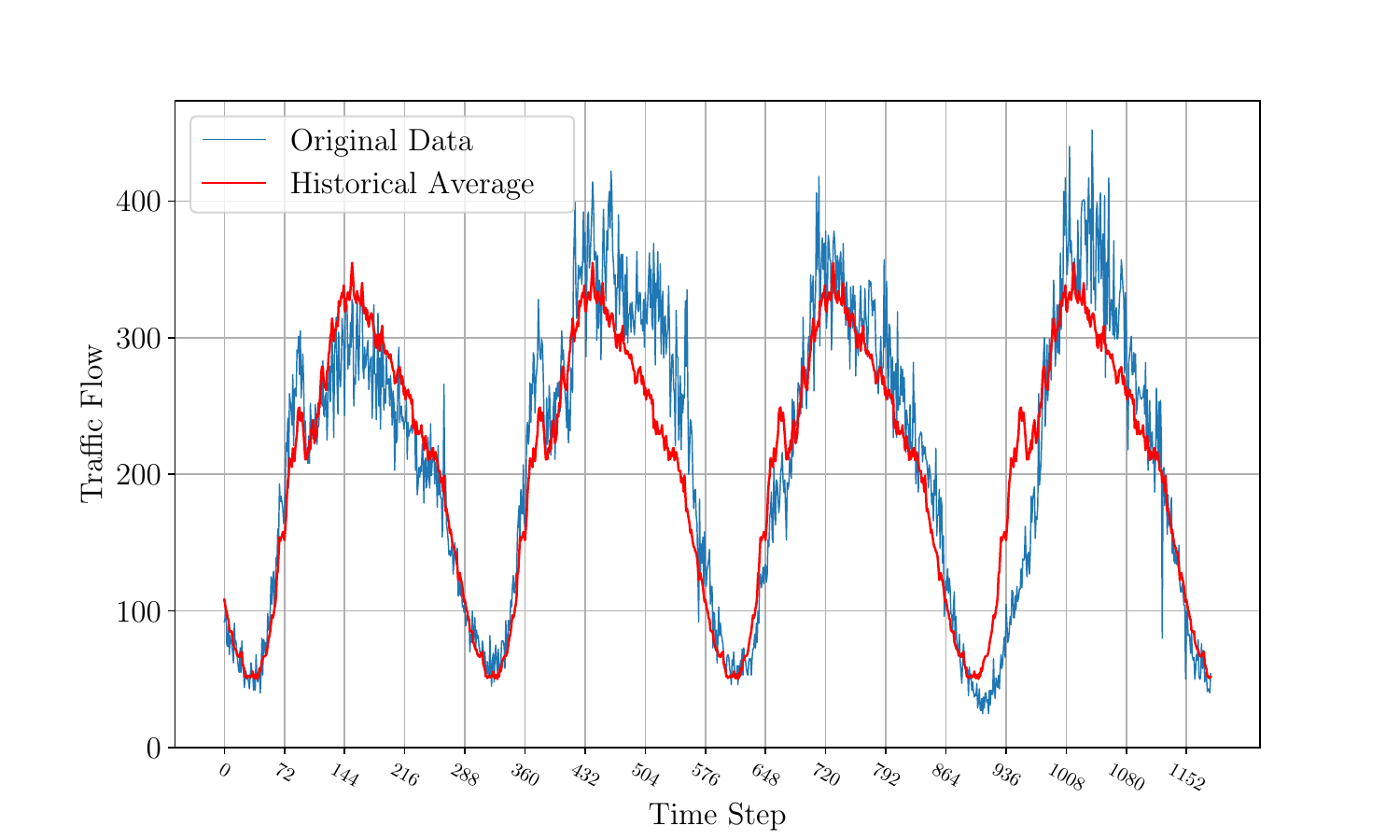}
    \caption{Example of traffic flow data and its historical average from sensor 120 in PEMS08 dataset.}
    \label{fig:ha_slice}
\end{figure}


\subsection{Non-parametric Regression based on Memory Bank}
\label{sec:np_reg}

Memory banks serve as dictionaries for feature retrieval and have demonstrated remarkable success in non-parametric tasks, such as 3D point cloud classification
Point-NN \cite{zhang2023parameter}. Point-NN determines the classification result by calculating the similarity score between the transformed input instance and training data.
Traffic data exhibits strong periodicity, making it feasible to predict future states by retrieving historical data \cite{cui2021historical}. Therefore, adopting the non-parametric framework and predicting future states by aggregating historical data is a feasible direction.

For a classification task, obtaining the prediction by selecting the highest total scores is intuitive and straightforward. In contrast, time series forecasting requires a sequence of continuous values as its output. To produce a sequence of continuous values, the similarity scores can be treated as the coefficients for calculating the approximation of the input instance by the training data,
\begin{equation} \label{equ:x_approx}
    \tilde{X}_i = \sum_{j\in \mathcal{T}} \alpha_{ij} X_j,
\end{equation}
where $\alpha_{ij}$ denotes the similarity score of $j$-th training data whose calculation formula is introduced later. By the assumption of strong periodicity of the traffic flow, we can infer that $X_{p(i)}^{imp}$ in Equation (\ref{equ:xi_def}) predominantly determines the pattern of $X_i$, and $Y^{imp}_{p(i)}$ determines the pattern of $Y_i$. Then we can approximate $Y_i$ by $\sum_{j\in \mathcal{T}} \alpha_{ij} Y_j$.

\noindent \textbf{Similarity Score Calculation.} Some studies have made efforts to the coefficient calculation of time series data \cite{mueen2010fast, adhikari2012combining}. In this framework based on a memory bank, the instance-wise similarity scores are calculated across the entire training dataset for each input instance. Consequently, computational complexity becomes critical, making optimization-based methods inefficient. In addition, the similarity score needs to be restricted to prevent overfitting. Therefore, we proposed a method to calculate the similarity score based on Euclidean distance with scaling. For an instance $X_i$, we first calculate and normalize its instance-wise distances to each training data,
\begin{equation} \label{equ:d_euc}
    D_{i} = \Big\{\Vert X_i - X_j\Vert_2 \Big\}_{j\in \mathcal{T}},
\end{equation}
\begin{equation} \label{equ:d_norm}
    \hat{D}_i = \frac{D_{i} - \min(D_i)}{\max(D_i) - \min(D_i)}.
\end{equation}
Inspired by the classification head of Point-NN, we scale the similarity score as follows,
\begin{equation} \label{equ:d_act}
    {\alpha}_{ij} = \exp(-(\gamma \hat{d}_{ij})^\beta),
\end{equation}
\begin{equation} \label{equ:a_norm}
    \hat{\alpha}_{ij} = \frac{\alpha_{ij}}{\sum_{k \in \mathcal{T}} \alpha_{ik}},
\end{equation}
where $\gamma$ and $\beta$ are manually defined hyperparameters for scaling and temperature, $\hat{d}_{ij} \in \hat{D}_{i}$. Equation (\ref{equ:d_act}) serves as an activation function that enhances the high similarity. Additionally, since $0 \le \hat{d}_{ij} \le 1$ (Equation (\ref{equ:d_norm})), the maximum value of $\hat{\alpha}_j$ is $1$, which constrains the value of each score. Using the calculated similarity scores, we can make a rough approximation of the prediction target,
\begin{equation} \label{equ:yp_sum}
    \hat{Y}_i = \sum_{j\in \mathcal{T}} \hat{\alpha}_{ij} Y_j.
\end{equation}

\noindent \textbf{Relation to Nadaraya–Watson kernel regression.} Ignoring the Equation (\ref{equ:d_norm}), the calculation of the similarity score can be represented as,
\begin{equation}
    \alpha^\prime_{ij} = \exp (-(\gamma \Vert X_i - X_j\Vert_2^\beta)).
\end{equation}
Let $\gamma=1/(2\sigma^2)$ and $\beta = 2$, we have,
\begin{align}
    \alpha^\prime_{ij} &= \exp (-\frac{\Vert X_i - X_j\Vert_2^2}{2\sigma^2}) \\
    &= K(X_i,X_j),
\end{align}
where $K(\cdot)$ denotes a standard Gaussian kernel. By combining Equation (\ref{equ:a_norm}), (\ref{equ:yp_sum}) with $K$, we have the prediction,
\begin{equation}
    \hat{Y}^{\prime}_i = \sum_{j\in \mathcal{T}} \frac{K(X_i,X_j)Y_j}{\sum_{k\in \mathcal{T}}K(X_i,X_k)},
\end{equation}
which is in the form of Nadaraya–Watson kernel regression with a Gaussian kernel. Therefore, we can infer that the proposed procedure of Equation (\ref{equ:d_euc}) to (\ref{equ:yp_sum}) has similar consistency with the Nadaraya–Watson estimator.


\subsection{Decoupling Time Series} \label{sec:dec_ts}


To handle the sparsity and high dimensionality of time series data, we require the kernel to aggregate information from the whole training set rather than relying on nearest neighbors. As a non-parametric method similar to Nadaraya-Watson kernel regression, this inevitably leads to oversmoothed prediction or underfitting \cite{gyorfi2006distribution}. To resolve this problem, we adopt the gradient boosting paradigm. 
We treat each non-parametric regression layer as a weak learner. By stacking multiple layers, each subsequent layer fits the residual errors of the previous one. This transforms a sequence of underfitting non-parametric regression layers into a precise strong model.
We utilize this idea to extract the residual signal from the input signal. Reusing the procedure of Equation (\ref{equ:d_euc}) to (\ref{equ:yp_sum}), the residual signal can be further processed by stacking layers. Let $f(\cdot)$ denote the process from Equation (\ref{equ:d_euc}) to (\ref{equ:a_norm}), we have,
\begin{equation} \label{equ:alpha_i}
    \hat{\boldsymbol{\alpha}}^{(\ell)}_i = f\Big(X^{(\ell)}_i, \{X^{(\ell)}_j\}_{j\in \mathcal{T}_i^{(\ell)}}\Big),
\end{equation}
\begin{equation} \label{equ:x_decouple}
    X^{(\ell+1)}_{i} = X^{(\ell)}_{i} - \sum_{j\in \mathcal{T}_i^{(\ell)}}\alpha^{(\ell)}_{ij} X_j^{(\ell)},
\end{equation}
where $X^{(\ell+1)}_i$ denotes the residual signal as the input of $(\ell + 1)$-th layer and $\mathcal{T}_i^{(\ell)}$ denote the range of training set for this layer. 
When $\ell > 1$, $X_j^{(\ell)}$ refers to the residual signal of a training sample. By avoiding calculating self-similarity ($j \notin \mathcal{T}_j$), the residual signals of training set can be calculated as follows,
\begin{equation} \label{equ:bank_a}
    \hat{\boldsymbol{\alpha}}^{(\ell)}_j = f\Big(X^{(\ell)}_j, \{X^{(\ell)}_k\}_{k\in \mathcal{T}_j^{(\ell)}}\Big),
\end{equation}
\begin{equation}
    X^{(\ell+1)}_{j} = X^{(\ell)}_{j} - \sum_{k\in \mathcal{T}_j^{(\ell)}}\alpha^{(\ell)}_{jk} X_k^{(\ell)},
\end{equation}
\begin{equation} \label{equ:bank_y}
    Y^{(\ell+1)}_{j} = Y^{(\ell)}_{j} - \sum_{k\in \mathcal{T}_j^{(\ell)}}\alpha^{(\ell)}_{jk} Y_k^{(\ell)},
\end{equation}
where $j$ and $k$ denote the indexes in the training set. 
To accelerate the inference, we calculate the residual signals in advance and store them in a memory bank. Let $L$ denote the number of layers, each entry $b_j$ stores a set of constructed pairs of residual signals, expressed as $b_j = \{ \langle X_j^{(\ell)}, Y_j^{(\ell)}\rangle\}_\ell^{L}$.
Based on the memory bank, the layer prediction and final prediction can be calculated as follows,
\begin{equation} \label{equ:yl_sum}
    \hat{Y}^{(\ell)}_i = \sum_{j\in \mathcal{T}_i^{(\ell)}}\alpha^{(\ell)}_{ij} Y_j^{(\ell)},
\end{equation}
\begin{equation} \label{equ:sum_pred}
    \hat{Y} = \sum_{\ell} \hat{Y}^{(\ell)}_i.
\end{equation}
Equation (\ref{equ:alpha_i}) to (\ref{equ:sum_pred}) form the basic structure of the proposed framework.

\noindent \textbf{Decoupling Periodic Information.} Recall the assumption of Equation (\ref{equ:xi_def}), we model that the time series data consists of a period-based implicit pattern dominating the entire pattern, and a perturbation caused by external factors. When the temporal label is provided, leveraging the periodic information benefits the accuracy of the approximation. Therefore, at the first layer, the input performs a \textit{time-wise decoupling}, which only matches the range of the memory bank that has the same periodic time step as the input, 
\begin{equation} \label{equ:t_range}
    \mathcal{T}^{(1)}_i = \Big\{j \in \mathcal{T}\backslash\{i\} \Big\vert  p(j) = p(i) \Big\},
\end{equation}
where $p(i)$ denotes the periodic time step of $i$.

\noindent \textbf{Decoupling Residual Signals.} We first prove that the periodic information can be decoupled in the first layer.

\begin{lemma} \label{thm:time_dec}
    The implicit pattern $X^{imp}_{p(i)}$ in Equation (\ref{equ:xi_def}) is decoupled from $X^{(1)}_i$ after being processed by Equation (\ref{equ:x_decouple}), i.e., $X^{(\ell)}_i$ is uncorrelated to $X^{imp}_{p(i)}$ for $\ell > 1$.
\end{lemma}

\begin{proof}
    By Equation (\ref{equ:x_decouple}), we have,
    \begin{align*}
        X^{(2)}_i &= X^{(1)}_i - \sum_{j\in \mathcal{T}^{(1)}_i} \alpha^{(1)}_{ij}X^{(1)}_j \\
        &= X^{imp}_{p(i)} + X^{dif}_i - \sum_{j\in \mathcal{T}^{(1)}_i} \alpha^{(1)}_{ij}(X^{imp}_{p(j)}+X^{dif}_j).
    \end{align*}
    Since Equation (\ref{equ:t_range}) specifies that $p(j) = p(i)$ and Equation (\ref{equ:a_norm}) implies that $\sum_j\alpha_j = 1$, we have,
    \begin{align*}
        X^{(2)}_i &= X^{dif}_i - \sum_{j\in \mathcal{T}^{(1)}_i} \alpha^{(1)}_{ij}X^{dif}_j.
    \end{align*}
    Therefore, $X^{(2)}_i$ is not correlated to $X^{imp}_{p(i)}$ anymore. Iteratively, $X^{(\ell)}_i$ is uncorrelated to $X^{imp}_{p(i)}$ for $\ell > 1$.
\end{proof}
Without the implicit pattern, the range of the matching memory bank is extended to include the entire training time steps within the corresponding layer, i.e., $\mathcal{T}^{(\ell)}_i = \mathcal{T}\backslash\{i\}$ when $\ell > 1$. 

In traffic forecasting, the predictable time series sequence commonly shows a continuous and tractable pattern \cite{cui2021historical}. It is reasonable to posit that the implicit pattern is continuous. Consequently, we can infer that the perturbation is also continuous, as the external factors change gradually. Because the residual signals are derived from the perturbations, the basic pattern of the historical data benefits the approximation of future data by focusing on the trends. \textit{Mean decoupling} module acts as a normalization technique to normalize the data distribution by subtracting the mean of the historical residual signal from both historical and future residual signals, ensuring the regression focuses on the trend of the residual signals across the training set. 
When $\ell > 1$, Equation (\ref{equ:bank_a}) to (\ref{equ:bank_y}) for memory bank construction are modified to,
\begin{equation} \label{equ:bank_a2}
    \hat{\boldsymbol{\alpha}}^{(\ell)}_j = f\Big(X^{(\ell)}_j - \bar{X}^{(\ell)}_j, \{X^{(\ell)}_k- \bar{X}^{(\ell)}_k\}_{k\in \mathcal{T}_j^{(\ell)}}\Big),
\end{equation}
\begin{equation}
    X^{(\ell+1)}_{j} = X^{(\ell)}_{j} - \bar{X}^{(\ell)}_j - \sum_{k\in \mathcal{T}_j^{(\ell)}}\alpha^{(\ell)}_{jk} (X_k^{(\ell)}-\bar{X}^{(\ell)}_k),
\end{equation}
\begin{equation} \label{equ:bank_y2}
    Y^{(\ell+1)}_{j} = Y^{(\ell)}_{j} - \bar{X}^{(\ell)}_j - \sum_{k\in \mathcal{T}_j^{(\ell)}}\alpha^{(\ell)}_{jk} (Y_k^{(\ell)}-\bar{X}^{(\ell)}_k).
\end{equation}
Similarly, the process of prediction is modified to,
\begin{equation}
    \hat{\boldsymbol{\alpha}}^{(\ell)}_i = f\Big(X^{(\ell)}_i - \bar{X}^{(\ell)}_i, \{X^{(\ell)}_j- \bar{X}^{(\ell)}_j\}_{j\in \mathcal{T}_i^{(\ell)}}\Big),
\end{equation}
\begin{equation}
    X^{(\ell+1)}_{i} = X^{(\ell)}_{i} - \bar{X}^{(\ell)}_i - \sum_{j\in \mathcal{T}_i^{(\ell)}}\alpha^{(\ell)}_{ij} (X_j^{(\ell)}-\bar{X}^{(\ell)}_j),
\end{equation}
\begin{equation}
    \hat{Y}^{(\ell)}_i = \bar{X}^{(\ell)}_i + \sum_{j\in \mathcal{T}_i^{(\ell)}}\alpha^{(\ell)}_{ij} (Y_j^{(\ell)}-\bar{X}^{(\ell)}_j),
\end{equation}
where $\bar{X}_i$ denotes the mean of $X_i$.

\subsection{Interpretability}

Due to its simple and non-parametric structure, TSNN theoretically does not have any black-box component. Additionally, by calculating instance-wise similarity scores through the memory bank, the model can provide the instance-wise contribution of each entry in the memory bank to the final prediction. We use the similarity scores calculated in the previous sections as the basic contribution reference. For better visualization, the similarity scores are extracted before they have been normalized (Equation (\ref{equ:a_norm})). We define the layer contribution of an entry $b_j$, which contains all pairs of $\langle X_j^{(\ell)}, Y_j^{(\ell)}\rangle$, as the similarity score $\alpha_{ij}^{(\ell)}$ of this layer multiplied by the mean of the layer prediction $\hat{Y}_i^{(\ell)}$. Note that in the first layer, the entries without a matched periodic time step to the input instance are assigned zero contribution in this layer. Therefore, the contribution of an entry $b_j$ can be expressed as follows,
\begin{equation}
    \mathrm{Contribution}(b_j) = \sum_\ell \alpha_{ij}^{(\ell)} \bar{Y}_i^{(\ell)},
\end{equation}
where $\bar{Y}_i^{(\ell)}$ denotes the mean of $\hat{Y}_i^{(\ell)}$.

\section{Experimental Results} \label{sec:exper}

In this section, we present the experiments to demonstrate the performance of TSNN for traffic flow forecasting and the effectiveness of its components on four real-world datasets. Section \ref{sec:exp_set} introduces the experimental setup for the following experiments, including selected datasets, baseline models, metrics, and implementation details. Then, the performance comparison among TSNN and baselines is depicted in Section \ref{sec:main_res}. We conduct experiments to verify the effectiveness of the components in Section \ref{sec:abl_study} and Section \ref{sec:hyp_study}. Section \ref{sec:data_vis} visualizes the data, and Section \ref{sec:interp} demonstrates the interpretability of TSNN. Finally, Section \ref{sec:effi_stu} reports the memory usage and processing time of TSNN.

\subsection{Experimental Setup} 
\label{sec:exp_set}
\subsubsection{Datasets}
To evaluate the performance of TSNN, we conduct experiments on four real-world benchmark datasets for traffic flow, including PEMS03, PEMS04, PEMS07, and PEMS08 \cite{song2020spatial}. The datasets contain the traffic flow data collected by sensors with a fine-grained time resolution of 5 minutes. 
Besides the traffic datasets, we also conduct experiments on three non-traffic datasets, including ETTh1, ETTh2 \cite{zhou2021informer} for the electric indicator, and BeijingAirQuality \cite{basicts} for the air quality index. 
The details of the datasets are listed in Table \ref{tab:ds_detail}.
Note that each dataset contains multiple sensors, TSNN processes each sensor as an independent univariate time series and calculates the average metrics.


\begin{table}[htbp]
\caption{Summary of datasets.}
\label{tab:ds_detail}
\renewcommand\arraystretch{1}
\centering
\scalebox{0.88}{
\begin{tabular}{c|c|c|c|c|c}
\toprule

{\textbf{Dataset}} & \textbf{\#Steps} & \textbf{\#Sensors} & \textbf{Frequency} & \textbf{Duration} & \textbf{Graph} \tabularnewline
\midrule

\textbf{PEMS03}           & 26208                  & 358                 & 5 mins                 & 3 months & Yes          \tabularnewline
\midrule
\textbf{PEMS04}           & 16992                  & 307                 & 5 mins                 & 2 months & Yes           \tabularnewline
\midrule
\textbf{PEMS07}           & 28224                  & 883                 & 5 mins                 & 4 months & Yes           \tabularnewline
\midrule
\textbf{PEMS08}           & 17856                  & 170                 & 5 mins                 & 2 months & Yes           \tabularnewline
\midrule
\textbf{ETTh1}           & 14400                  & 7                & 60 mins                 & 20 months & No           \tabularnewline
\midrule
\textbf{ETTh2}           & 14400                  & 7                & 60 mins                 & 20 months & No           \tabularnewline
\midrule
\textbf{BeijingAirQuality}           & 36000                  & 7                & 60 mins                 & 50 months & No           \tabularnewline
\bottomrule
\end{tabular}
}
\end{table}

\subsubsection{Baselines}
We compare the proposed TSNN with a comprehensive set of baselines. 
The traditional methods include Vector Auto-Regression (VAR) \cite{lu2016integrating} and Historical Inertia (HI) \cite{cui2021historical}.
The representative graph-based models include STGCN \cite{yu2017spatio}, DCRNN \cite{li2018diffusion}, and Graph WaveNet \cite{wu2019graph}.
Recent efficient baselines include DLinear \cite{zeng2023transformers}, FITS \cite{xufits}, RLinear \cite{li2023revisiting}, DeepNPTS \cite{rangapuram2023deep}, and STID \cite{shao2022spatial}. Other baselines include a transformer-based model Informer \cite{zhou2021informer} and a spatial-temporal model ST-Norm \cite{deng2021st}. Note that the graph-based models (STGCN, DCRNN, Graph WaveNet) are not evaluated on non-traffic datasets as the graphs are not provided.



\subsubsection{Metrics}
The evaluation employs Mean Absolute Error (MAE), Root Mean Square Error (RMSE), and Mean Absolute Percentage Error (MAPE) as the metrics. The formulas of metrics are as follows,
\begin{equation}
    \mathrm{MAE} = \frac{1}{\vert \Omega \vert} \sum_{m \in \Omega} \vert y_m - \hat{y}_m \vert,
\end{equation}
\begin{equation}
    \mathrm{RMSE} = \sqrt{\frac{1}{\vert \Omega \vert} \sum_{m \in \Omega} (y_m - \hat{y}_m)^2},
\end{equation}
\begin{equation}
    \mathrm{MAPE} = \frac{1}{\vert \Omega \vert} \sum_{m \in \Omega} \Big\vert \frac{y_m - \hat{y}_m}{y_m} \Big\vert,
\end{equation}
where $y_m$ denote the $m$-th value in the ground truth $Y$ and $\hat{y}_m$ denote the $m$-th value in the prediction $\hat{Y}$, $\Omega$ is the indices of predicted samples where $\vert \Omega \vert = T^\prime$ where $T^\prime$ is the number of time steps in future data.

\subsubsection{Implementation Details}\label{sec:impdetails}

The implementation is based on BasicTS \cite{shao2024exploring} that provides a fair benchmark and toolkit for time series forecasting. Following the previous works \cite{song2020spatial} and the default settings in BasicTS, we divide the datasets into training, validation, and test sets with a ratio of 6:2:2. 
The time steps of historical and future data are set to 12, i.e., $T = T^\prime=12$, and this setting is maintained across all datasets for consistency. We set the number of layers $L = 10$. 
Regarding hyperparameter settings, 
we conducted a grid search on the validation set of the PEMS08 dataset to identify the optimal configuration. This selected configuration was then applied to all datasets to demonstrate generalization. 
Consequently, we set $\gamma = 10$, $\beta=1.5$ and a 3-step tolerance in the time-wise matching (i.e., $\vert p(i) - p(j)\vert \le 3$) in the first layer.
The proposed model is implemented by PyTorch 2.5.1, and the experiments are conducted with an NVIDIA GeForce RTX 4090 GPU. Our code is available at \url{https://github.com/pzzzzzm/TSNN_release}.

\begin{table*}
    \renewcommand\arraystretch{1.2}
    \centering

    \caption{Performance comparison on PEMS03, PEMS04, PEMS07, PEMS08, ETTh1, ETTh2, and BeijingAirQuality datasets. The \textit{Rank} column indicates the ranking of our proposed TSNN against baseline models for each metric.}
    \label{tab:all_perf}
    \scalebox{0.73}{
    \small
    \begin{tabular}{cc|cccccccccccc|cc}
    \toprule
    \midrule
    \textbf{Dataset} & \textbf{Metric} & \textbf{HI} & \textbf{VAR} & \textbf{DeepNPTS} & \textbf{STGCN} & \textbf{DCRNN} & \textbf{GWNet} & \textbf{DLinear} & \textbf{FITS} & \textbf{RLinear} & \textbf{Informer} & \textbf{STNorm} & \textbf{STID} & \textbf{TSNN} & \textit{Rank} \\
    \midrule
    \midrule
    \multirow{3}{*}{PEMS03} & MAE & 32.40 & 22.04 & 21.41 & 15.71 & 15.50 & \textbf{14.45} & 21.08 & 23.61 & 21.60 & 19.11 & \underline{15.03} & 15.37 & 15.21 & 3 \\
     & RMSE & 49.73 & 35.97 & 34.52 & 28.07 & 26.84 & 25.25 & 33.90 & 37.69 & 35.02 & 33.20 & 26.26 & \underline{25.1} & \textbf{24.44} & \textbf{1} \\
     & MAPE & 30.59\% & 22.81\% & 20.41\% & 15.40\% & 15.49\% & \underline{14.82\%} & 20.63\% & 22.45\% & 19.64\% & 18.35\% & \textbf{14.31\%} & 15.89\% & 15.81\% & 5 \\
    \midrule
    \multirow{3}{*}{PEMS04} & MAE & 42.36 & 23.51 & 27.96 & 19.63 & 19.71 & 18.97 & 27.34 & 30.14 & 28.28 & 22.26 & \underline{18.96} & \textbf{18.68} & 19.28 & 4 \\
     & RMSE & 61.66 & 36.39 & 42.99 & 31.32 & 31.43 & \textbf{30.32} & 42.75 & 46.16 & 44.37 & 36.14 & 30.98 & \underline{30.61} & 31.69 & 6 \\
     & MAPE & 29.92\% & 17.85\% & 19.27\% & 13.32\% & 13.54\% & 14.26\% & 19.45\% & 21.08\% & 18.59\% & 15.70\% & \underline{12.69\%} & \textbf{12.48\%} & 13.03\% & 3 \\
    \midrule
    \multirow{3}{*}{PEMS07} & MAE & 49.04 & 37.06 & 32.15 & 21.71 & 21.20 & \underline{20.25} & 30.87 & 38.07 & 32.03 & 31.76 & 20.52 & \textbf{20.07} & 20.37 & 3 \\
     & RMSE & 71.18 & 55.73 & 49.5 & 35.41 & 34.43 & \textbf{33.32} & 47.99 & 57.43 & 50.07 & 56.49 & 34.85 & \underline{33.39} & 34.56 & 4 \\
     & MAPE & 22.75\% & 19.93\% & 14.49\% & 9.25\% & 9.06\% & 8.63\% & 13.38\% & 17.65\% & 13.51\% & 13.98\% & 8.77\% & \textbf{8.40\%} & \underline{8.59\%} & \underline{2} \\
    \midrule
    \multirow{3}{*}{PEMS08} & MAE & 34.57 & 22.07 & 22.81 & 15.98 & 15.26 & 14.67 & 21.86 & 23.70 & 22.61 & 22.02 & 15.54 & \textbf{14.55} & \underline{14.57} & \underline{2} \\
     & RMSE & 50.43 & 31.02 & 35.17 & 25.37 & 24.28 & \textbf{23.49} & 34.41 & 36.76 & 35.81 & 35.14 & 25.01 & \underline{23.85} & 24.80 & 4 \\
     & MAPE & 21.63\% & 14.04\% & 14.38\% & 10.43\% & 9.96\% & \underline{9.52\%} & 13.63\% & 14.69\% & 13.64\% & 13.10\% & 10.03\% & \textbf{9.34\%} & 9.58\% & 3 \\
    \midrule
    \multirow{3}{*}{ETTh1} & MAE & 3.20 & 2.28 & 1.92 & - & - & - & 1.96 & 1.88 & 1.82 & 1.87 & 1.78 & \underline{1.64} & \textbf{1.53} & \textbf{1} \\
     & RMSE & 6.50 & 3.96 & 4.02 & - & - & - & 4.02 & 4.01 & 3.96 & 3.38 & \underline{3.32} & 3.50 & \textbf{3.16} & \textbf{1} \\
     & MAPE & 113.73\% & 67.72\% & 76.17\% & - & - & - & 75.85\% & 77.61\% & 74.19\% & 64.00\% & \underline{62.25\%} & 65.36\% & \textbf{61.86\%} & \textbf{1} \\
    \midrule
    \multirow{3}{*}{ETTh2} & MAE & 3.84 & 3.10 & 2.24 & - & - & - & 2.22 & 2.23 & \underline{2.15} & 2.18 & 2.40 & \textbf{2.04} & 2.21 & 4 \\
     & RMSE & 5.65 & 4.49 & 3.67 & - & - & - & 3.67 & 3.69 & 3.64 & \underline{3.35} & 3.81 & \textbf{3.34} & 3.48 & 3 \\
     & MAPE & 50.30\% & 50.99\% & 34.55\% & - & - & - & 32.35\% & 32.23\% & \textbf{28.49\%} & 30.08\% & 38.27\% & \underline{28.57\%} & 31.09\% & 4 \\
    \midrule
    \multirow{3}{*}{BJAirQuality} & MAE & 25.36 & 15.60 & 16.31 & - & - & - & 16.22 & 16.51 & 15.83 & \underline{14.38} & \textbf{13.9} & 14.82 & 15.41 & 4 \\
     & RMSE & 51.92 & 33.38 & 37.23 & - & - & - & 35.81 & 36.13 & 35.30 & 33.81 & \textbf{32.69} & \underline{33.08} & 34.59 & 5 \\
     & MAPE & 79.58\% & 64.86\% & 48.09\% & - & - & - & 52.44\% & 47.69\% & 46.26\% & \underline{46.03\%} & \textbf{44.62\%} & 48.06\% & 57.70\% & 8 \\
    \bottomrule
\end{tabular}
    }
\end{table*}

\subsection{Overall Performance} 
\label{sec:main_res}

In this section, we present a comprehensive performance comparison between TSNN and the baselines across multiple real-world datasets,  covering both traffic and non-traffic domains. The overall results are summarized in Table \ref{tab:all_perf}. 
We highlight the best results in bold and the second-best results in underlined fonts.
The \textit{Rank} column indicates the performance ranking of the proposed TSNN on each metric. 
Note that for non-traffic datasets (ETTh1, ETTh2, and BeijingAirQuality), graph-based models (STGCN, DCRNN, GWNet) are excluded from the comparison as these datasets do not provide a predefined spatial graph structure.

Despite being a non-parametric framework, TSNN demonstrates remarkable competitiveness against deep learning models. As indicated in the Rank column, TSNN achieves a top-3 ranking in 11 out of the 21 metrics across all datasets. On the traffic datasets, TSNN consistently outperforms statistical baselines and shows comparable performance to deep learning models. For example, TSNN achieves the best RMSE on PEMS03 and the second-best MAE on PEMS08. Compared to STID, which is a strong MLP-based baseline, TSNN still shows competitiveness, surpassing it in all metrics on PEMS03. On the non-traffic datasets, TSNN achieves the best performance on all metrics on the ETTh1 dataset. For ETTh2, TSNN remains competitive, securing a top-4 position in all metrics and showing stability in handling electricity data. Note that TSNN does not utilize the spatial information in the datasets, leaving opportunities for further enhancement.

\noindent \textbf{MAPE Instability on BeijingAirQuality.} Although TSNN ranks 4th and 5th in MAE and RMSE, respectively, on BeijingAirQuality dataset, its MAPE ranking drops to 8th. MAPE can be inflated by small values in the dataset, and the weighted aggregation mechanism of TSNN makes this sensitivity more significant. We verify this by calculating the Near-Zero Ratio for all datasets, which is defined as the average percentage of values of sensors falling within the bottom 5\% of the magnitude range (i.e., values $\le 0.05 \times \text{Max}$). The results are summarized in Table \ref{tab:nzr}.
\begin{table}[h]
\caption{Near-Zero Ratio of datasets.}
\label{tab:nzr}
\renewcommand{\arraystretch}{0.9}
\centering
\scalebox{0.9}{
\begin{tabular}{c|c}
\toprule
\textbf{Dataset} &  \textbf{Near-Zero Ratio} \\ 
\midrule
PEMS03 &  6.23\% \\
\midrule
PEMS04 & 6.71\% \\
\midrule
PEMS07 & 2.35\% \\
\midrule
PEMS08 & 1.58\% \\
\midrule
ETTh1 & 15.64\% \\
\midrule
ETTh2 & 24.75\% \\
\midrule
BeijingAirQuality & \textbf{34.37\%} \\
\bottomrule
\end{tabular}
}
\end{table}
The results reveal that BeijingAirQuality is an outlier with the highest Near-Zero Ratio of 34.37\%. This ratio is significantly higher than the other datasets. While TSNN can handle moderate near-zero values if they exhibit strong periodicity (as in ETTh1), the large proportion of small values in BeijingAirQuality makes the MAPE metric unstable for this dataset, despite TSNN achieving competitive MAE performance.

\subsection{Ablation Study} \label{sec:abl_study}

In this section, we conduct several ablation experiments to verify the effectiveness of the components in the proposed framework. We first explore the relation between the performance and the number of layers. Then we verify the effectiveness of the time-wise decoupling in the first layer and the mean decoupling operation in the other layers. Lastly, we verify the necessity of scaling the similarity score. 

\subsubsection{Number of Layers}
Table \ref{tab:layer_study} shows the performance of TSNN with different numbers of layers on the four datasets. We can observe that the errors decrease when the number of layers increases, which verifies the effectiveness of the stacked structure. The first five layers contribute most of the performance and the effect of the last five layers is comparatively minor. Therefore, the number of layers can be lower for efficiency in real applications.

\begin{table}[htbp]
\centering
\caption{Study on the number of layers of TSNN.}
\label{tab:layer_study}

\scalebox{0.73}{
\small
\begin{tabular}{cc|cccccc}

\toprule


& & \multicolumn{6}{c}{\textbf{Number of Layer(s) in TSNN}} \tabularnewline

\midrule

\textbf{Datasets} & \textbf{Metric} & \textbf{1 layer} & \textbf{2 layers} & \textbf{3 layers} & \textbf{5 layers} & \textbf{7 layers} & \textbf{10 layers} \tabularnewline


\midrule

\multirow{3}{*}{\textbf{PEMS03}} & MAE & 17.38 & 15.46 & 15.39 & 15.30 & 15.25 & 15.21 \tabularnewline
& RMSE & 30.14 & 24.79 & 24.68 & 24.56 & 24.49 & 24.44 \tabularnewline
& MAPE & 18.06\% & 16.01\% & 15.95\% & 15.88\% & 15.84\% & 15.81\% \tabularnewline

\midrule

\multirow{3}{*}{\textbf{PEMS04}} & MAE & 20.68 & 19.62 & 19.45 & 19.34 & 19.30 & 19.28 \tabularnewline
& RMSE & 33.75 & 31.98 & 31.87 & 31.76 & 31.72 & 31.69 \tabularnewline
& MAPE & 13.82\% & 13.28\% & 13.17\% & 13.07\% & 13.04\% & 13.03\% \tabularnewline

\midrule

\multirow{3}{*}{\textbf{PEMS07}} & MAE & 23.24 & 21.31 & 20.90 & 20.50 & 20.42 & 20.37 \tabularnewline
& RMSE & 39.56 & 35.20 & 34.95 & 34.73 & 34.63 & 34.56 \tabularnewline
& MAPE & 9.97\% & 9.13\% & 8.94\% & 8.66\% & 8.62\% & 8.59\% \tabularnewline

\midrule

\multirow{3}{*}{\textbf{PEMS08}} & MAE & 16.96 & 15.39 & 14.99 & 14.67 & 14.61 & 14.57 \tabularnewline
& RMSE & 27.33 & 25.31 & 25.08 & 24.92 & 24.85 & 24.80 \tabularnewline
& MAPE & 11.29\% & 10.02\% & 9.84\% & 9.64\% & 9.60\% & 9.58\% \tabularnewline

\bottomrule
\end{tabular}
}
\end{table}

\subsubsection{Effectiveness of Decoupling}
Fig. \ref{fig:abl} illustrates the comparison among the proposed TSNN and three variants of TSNN, including TSNN without time-wise decoupling, TSNN without mean decoupling, and TSNN without both modules. The proposed TSNN is represented by a blue circle line. The orange triangle and the red diamond line, representing the models without the time-wise decoupling module, start from high MAE in the first layer. This observation shows the critical role of periodic information and time-wise decoupling. The green line representing the TSNN without mean decoupling shows a slower fitting trend compared to the proposed TSNN. In the previous two lines above, the performance of TSNN without both modules surpasses that of TSNN without time-wise decoupling after three layers. This could be attributed to the mean decoupling losing some useful patterns from the time series. If additional layers are stacked, the performance of TSNN without mean decoupling might outperform the proposed TSNN. Considering the efficiency, the mean decoupling module is a practical trade-off to achieve satisfactory performance with much fewer layers. 

\begin{figure}
    \centering
    \includegraphics[width=1\linewidth, trim={50 10 60 50}, clip]{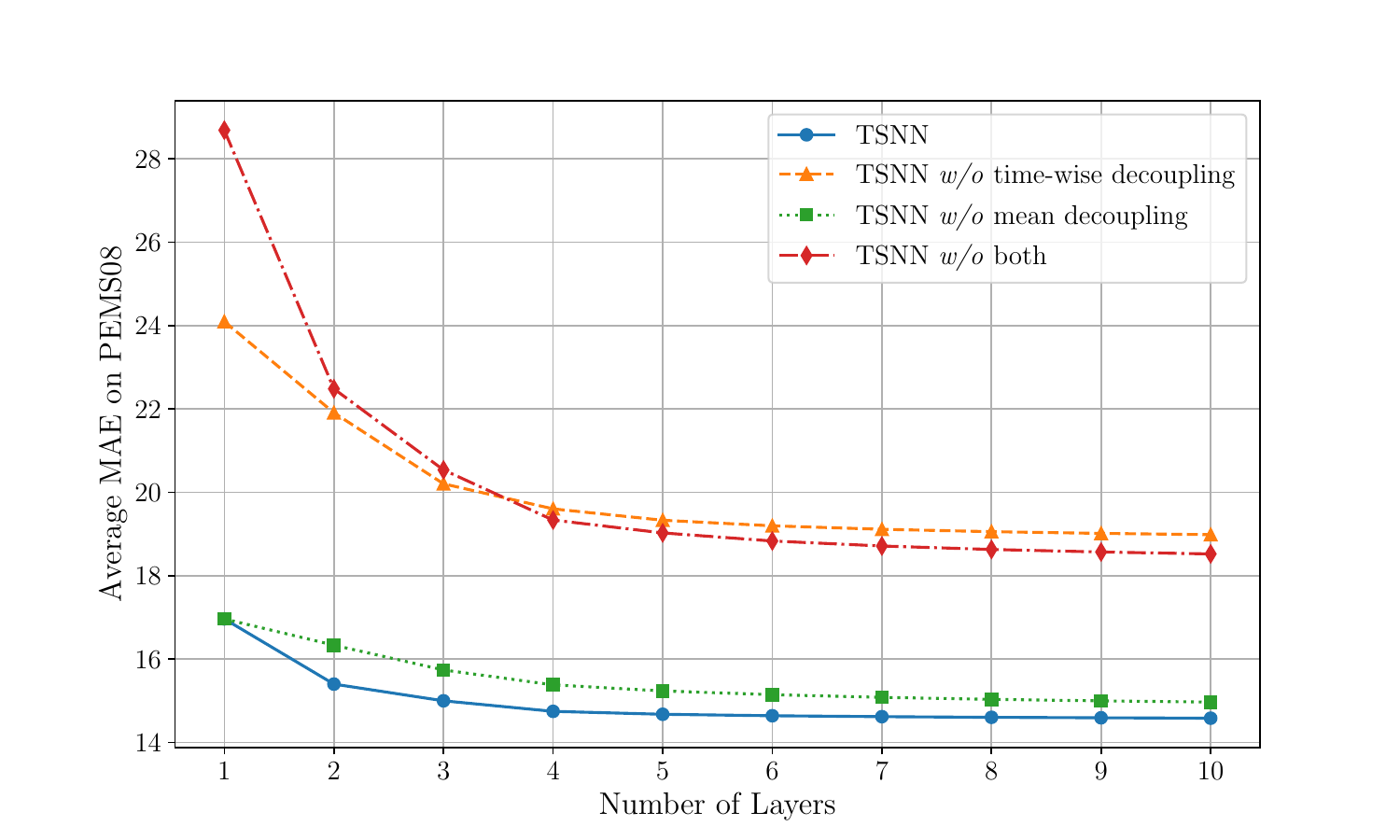}
    \caption{Ablation study of the time-wise decoupling and mean decoupling modules of TSNN. The performances of the proposed TSNN, TSNN without the time-wise decoupling, TSNN without the mean decoupling, and TSNN without both decoupling modules have been compared.}
    \label{fig:abl}
\end{figure}

\subsubsection{Similarity Score Scaling}
The scaling method of the similarity scores decides the degree of fitting for the layer prediction to the memory bank. The prediction result may be overfitting if the model assigns excessive weights to the entries with high similarity. We conduct an ablation experiment to evaluate the effect of the applied scaling method in Equation \ref{equ:d_act}. Considering the purpose of the scaling method, which assigns high scores to the entries with low distances, we compare the performance with the following baseline scaling methods:
\begin{itemize}
    \item No scaling (complement),
    \begin{equation}
        \alpha_{ij} = 1 - \hat{d}_{ij}.
    \end{equation}
    \item Inverse square scaling,
    \begin{equation}
        \alpha_{ij} = \frac{1}{d_{ij}^2 + \epsilon}.
    \end{equation}
    \item Sigmoid scaling,
    \begin{equation}
        \alpha_{ij} = \frac{1}{1+e^{d_{ij}-\mu}}.
    \end{equation}
\end{itemize}
We set $\epsilon=1e{-5}, \mu=0.5$. The experiment results on datasets PEMS04, PEMS08 and ETTh1 are shown in Table \ref{tab:scaling_abl}. The applied scaling method achieves the best performance on all the tested datasets, which verifies the necessity of the scaling method in Equation \ref{equ:d_act}.

\begin{table}[htbp]
\caption{Ablation study of the scaling methods of similarity scores.}
\label{tab:scaling_abl}
\renewcommand\arraystretch{1}
\centering
\scalebox{0.9}{
\begin{tabular}{l|c|c|c}
\toprule
\midrule
\multicolumn{1}{c|}{\textbf{Dataset}} & \textbf{PEMS04} & \textbf{PEMS08} & \textbf{ETTh1} \tabularnewline
\midrule
\midrule
\multicolumn{1}{c|}{\textbf{Method}} &  \textbf{MAE} &  \textbf{MAE} & \textbf{MAE} \tabularnewline
\midrule

TSNN ($w/o$ scaling) & 22.06 & 18.13  & 1.76 \tabularnewline
\midrule
TSNN ($w.$ inverse square scaling) & 26.84 & 19.15  & 2.20 \tabularnewline
\midrule
TSNN ($w.$ sigmoid scaling) & 24.02 & 19.85  & 1.93 \tabularnewline
\midrule
TSNN & \textbf{19.28} & \textbf{14.57} & \textbf{1.53} \tabularnewline
\bottomrule
\end{tabular}
}
\end{table}

\subsection{Hyperparameter Study} \label{sec:hyp_study}

\begin{figure}[!ht]
    \centering
    \subfloat[]{
    \includegraphics[width=0.9\linewidth, trim={13 10 53 35}, clip]{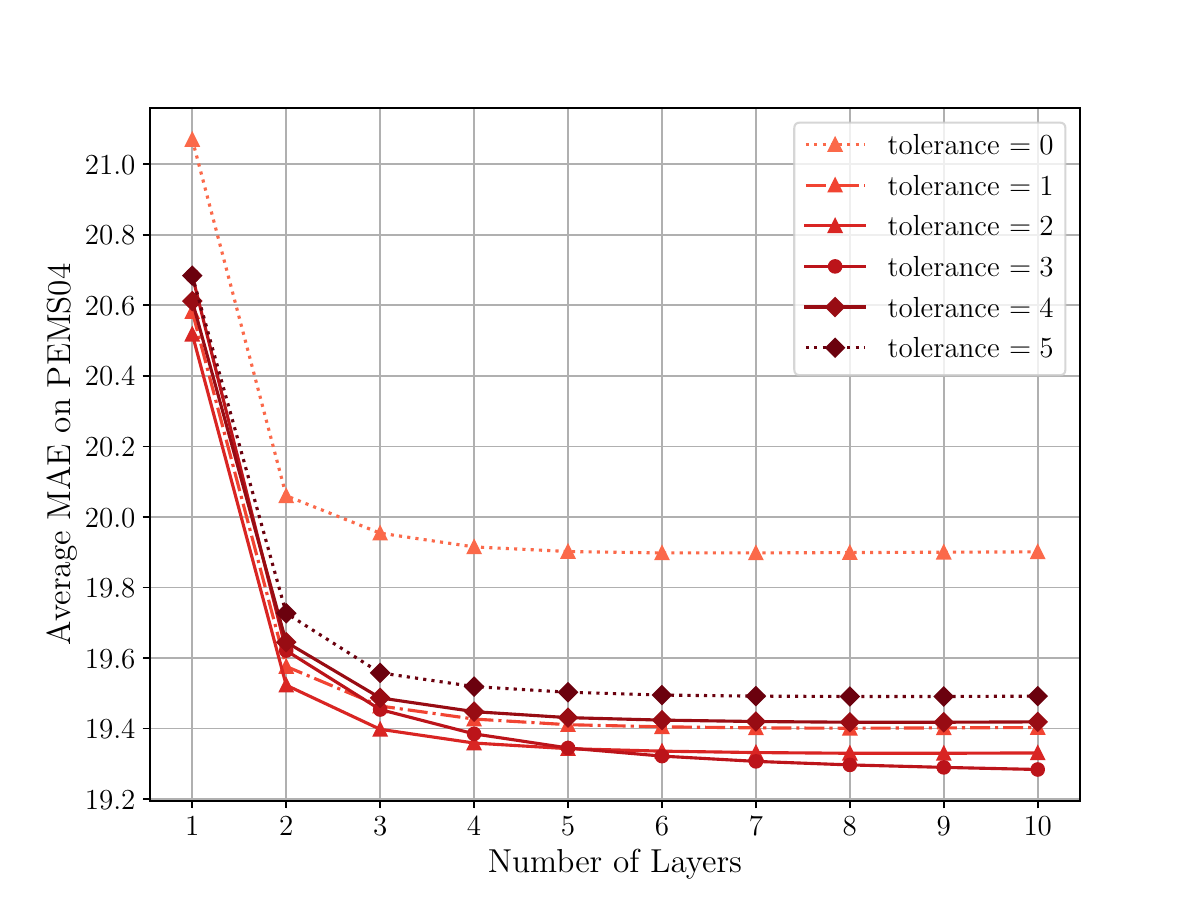}
    \label{fig:hyper_tol}
    }
    \vfill
    \subfloat[]{
    \includegraphics[width=0.9\linewidth, trim={13 10 53 35}, clip]{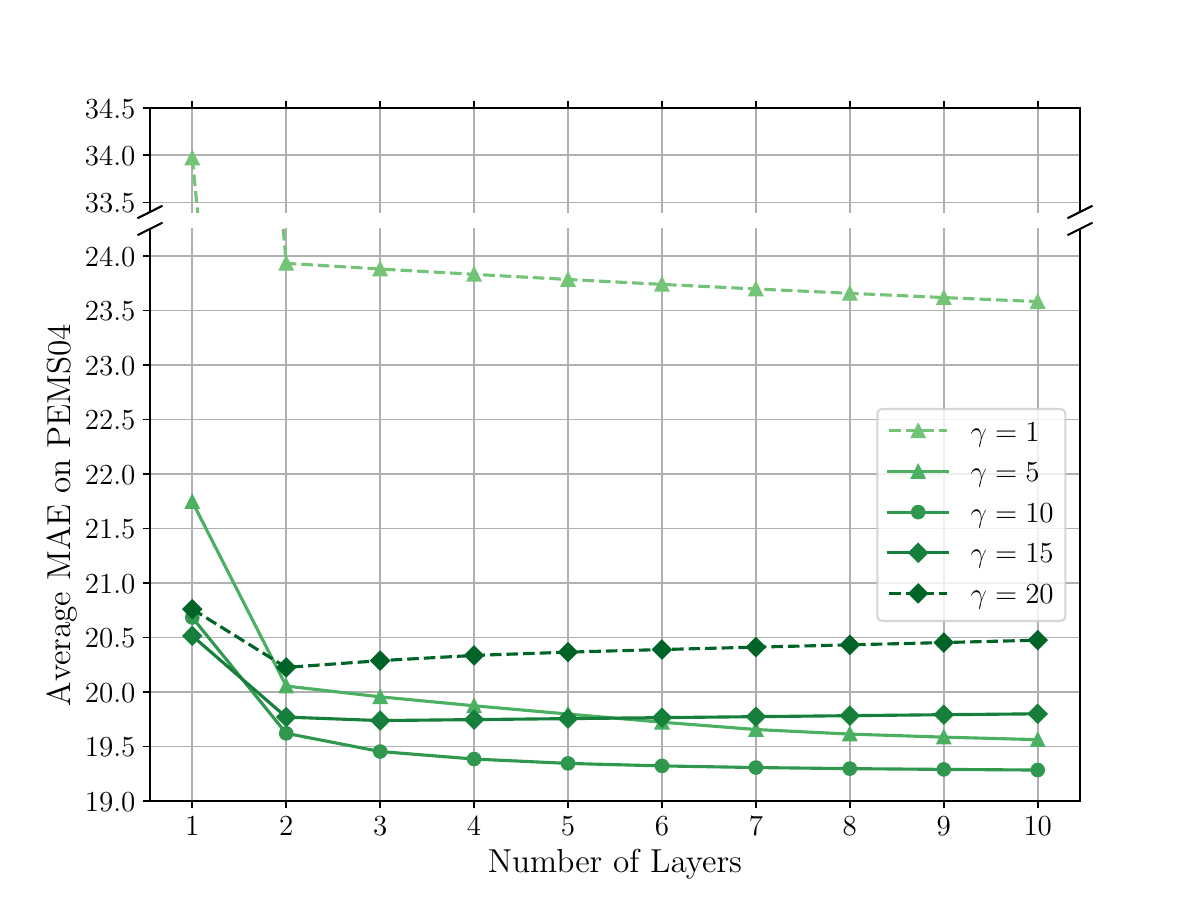}
    \label{fig:hyper_gamma}
    }
    \vfill
    \subfloat[]{
    \includegraphics[width=0.9\linewidth, trim={13 10 53 35}, clip]{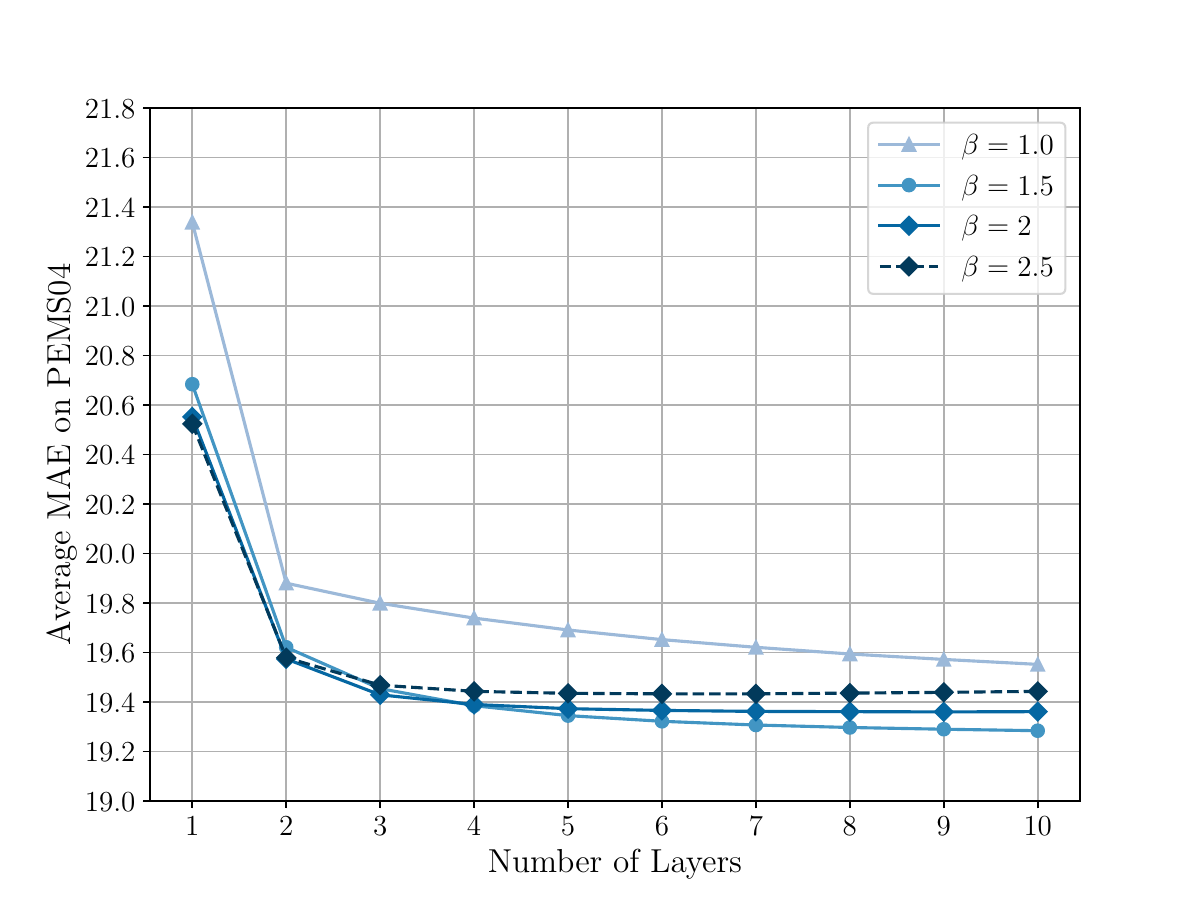}
    \label{fig:hyper_beta}
    }
    \caption{Hyperparameter studies of TSNN on PEMS04 where the lines with circle markers are the proposed settings. (a) The study on time-wise tolerance. (b) The study on $\gamma$. (c) The study on $\beta$. }
    \label{fig:hyper_study}
\end{figure}

This section conducts experiments to analyze the impacts of the hyperparameters, including time tolerance in time-wise decoupling and the values of $\gamma$ and $\beta$. The result of the experiment on the number of layers has been presented in Section \ref{sec:abl_study}. We conduct experiments for each hyperparameter by modifying the value of this hyperparameter and fixing other hyperparameters to the proposed settings on the PEMS04 dataset. The results are shown in Fig. \ref{fig:hyper_study}, where the lines with circle markers are the proposed setting. We observe that the performance of the proposed settings is surpassed by some settings in the early layers, such as lower tolerance and higher $\gamma$ and $\beta$. These settings have a common property that prefers assigning high similarity scores to fewer entries, leading to overfitting in the higher layers. In contrast, the reversed settings (i.e., higher tolerance and lower $\gamma, \beta$) average the scores and weaken the impact of the important entries, leading to underfitting. Furthermore, as detailed the unified hyperparameter configuration ($\gamma=10, \beta=1.5, \mathrm{tolerance} = 3$), which was selected by grid search on the PEMS08 validation set, consistently achieves optimal performance across the plots in Fig. \ref{fig:hyper_study}. This provides compelling evidence of the robustness and generalization capability of TSNN.

\subsection{Data Visualization} \label{sec:data_vis}

In this section, we visualize examples of prediction compared to the ground truth to investigate the proposed framework. Fig. \ref{fig:pred_vis} visualizes slices of predictions for sensors 125 and 126 on the PEMS08 dataset with corresponding ground truth, where each point represents the value of the 6th time step in each predicted time series. The index 0 represents the 979th time step in the test set as its periodic time step is 0. The prediction in Fig. \ref{fig:pred_125} for sensor 125 presents a perfect matching pattern. However, the prediction in Fig. \ref{fig:pred_126} for sensor 126 shows obvious flaws in each periodic series. 
This contrast highlights the framework's robust capability in capturing periodic patterns, while simultaneously revealing limitations in handling out-of-distribution patterns.

\begin{figure}[ht]
    \centering
    \subfloat[]{
    \includegraphics[width=1\linewidth, trim={10 15 0 15}, clip]{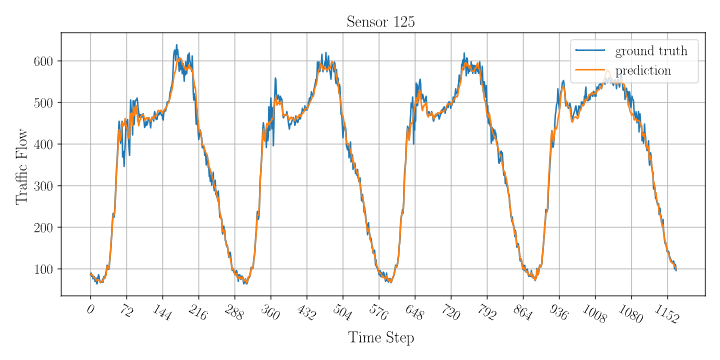}
    \label{fig:pred_125}
    }
    \vfill
    \subfloat[]{
    \includegraphics[width=1\linewidth, trim={10 15 0 15}, clip]{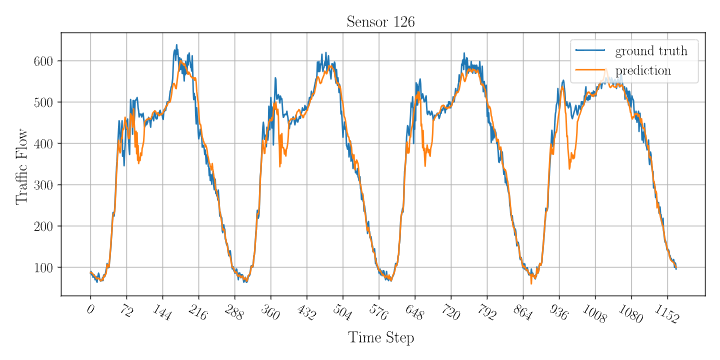}
    \label{fig:pred_126}
    }
    \caption{Visualization of prediction on PEMS08. The index 0 in the figures represents the time step 979 in the test set where the periodic time step is 0.  (a) Sensor 125. (b) Sensor 126.}
    \label{fig:pred_vis}
\end{figure}

To further investigate the behavior of the framework when encountering out-of-distribution patterns, we visualize the aggregated prediction after each layer in Fig. \ref{fig:vis_dec_tsnn}. The predictions of TSNN without mean decoupling are visualized in \ref{fig:vis_dec_nomean} to provide more insights. 
In Fig. \ref{fig:vis_dec_tsnn}, a notable divergence occurs between steps 108 and 216 but is effectively corrected by layer 2. The error span corresponds to the length of the input window plus the prediction horizon, highlighting the model's reliance on mean decoupling.
In Fig. \ref{fig:vis_dec_nomean}, the framework struggles but gradually fixes the prediction to approach the ground truth, which demonstrates the original idea of matching residual signals. This variant might achieve better predictions with an increased number of layers, which derives the same conclusion as Section \ref{sec:abl_study}.

\begin{figure}[ht]
    \centering
    \subfloat[]{
    \includegraphics[width=0.97\linewidth, trim={15 15 0 15}, clip]{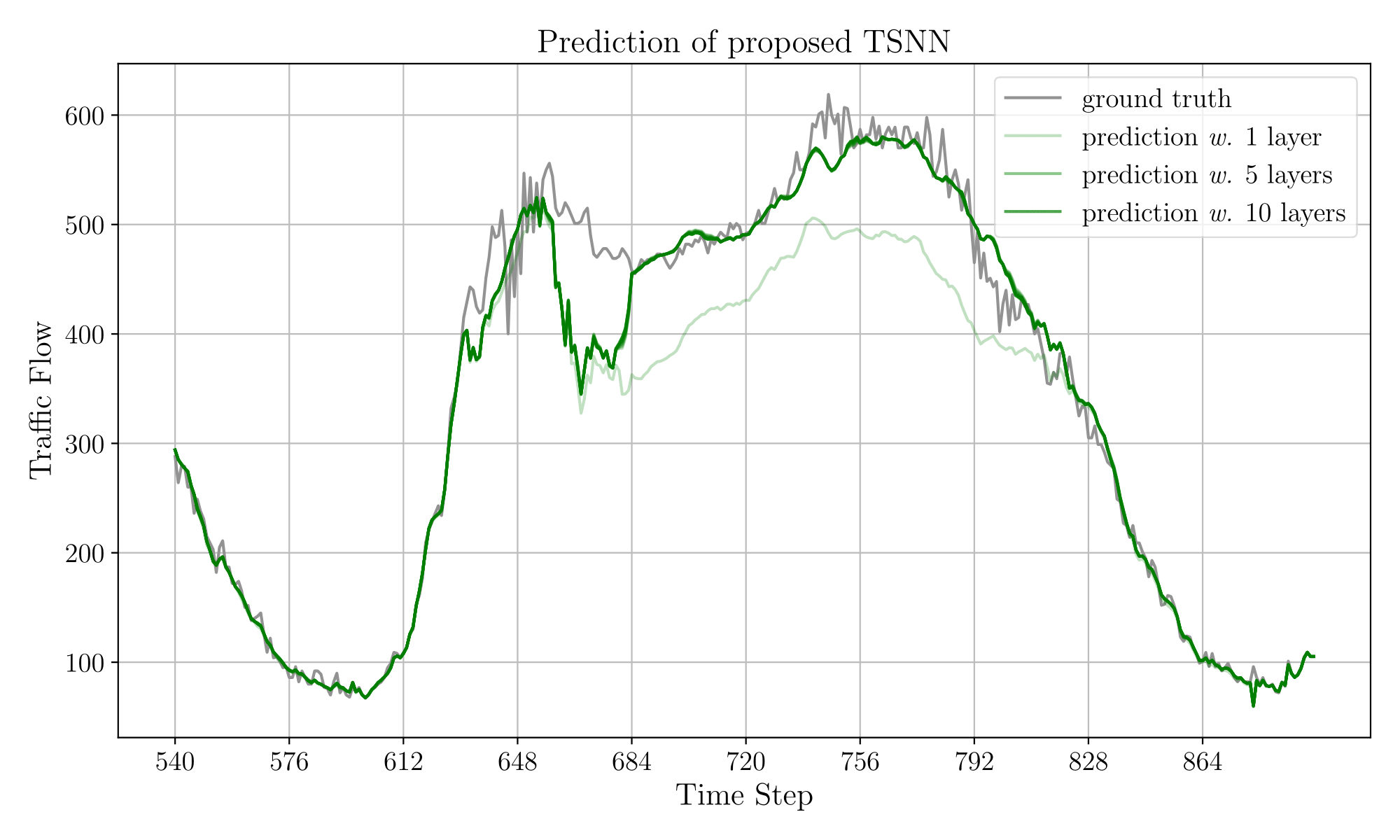}
    \label{fig:vis_dec_tsnn}
    }
    \vfill
    \subfloat[]{
    \includegraphics[width=0.97\linewidth, trim={15 15 0 15}, clip]{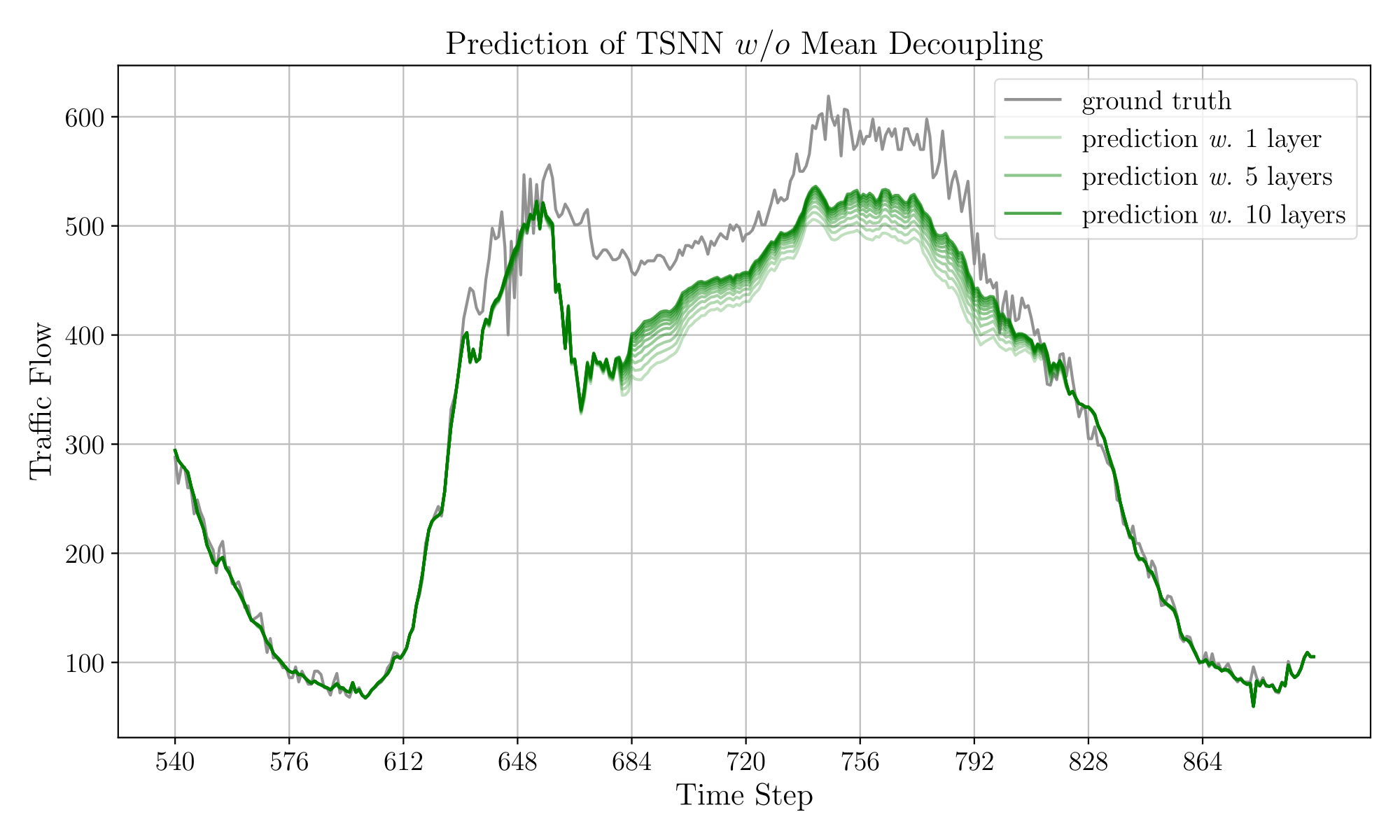}
    \label{fig:vis_dec_nomean}
    }
    \caption{Visualization of aggregated predictions after each layer on sensor 126 in PEMS08. (a) Predictions of original TSNN. (b) Predictions of TSNN without mean decoupling.}
    \label{fig:vis_dec}
\end{figure}

\noindent \textbf{Does the first layer decouple the periodic information?} 
To experimentally verify Lemma \ref{thm:time_dec}, we visualize the clustered input and output of the first layer. We utilize t-SNE \cite{van2008visualizing} to cluster the input time series and output residual signals of each time point in the training set and plot them as points in Fig. \ref{fig:tsne_input} and \ref{fig:tsne_output}. The periodic time step of each point is represented by the depth of color. We observe that in Fig. \ref{fig:tsne_input}, the points show a distinct distribution according to their periodic time step, whereas the pattern of distribution disappears in Fig. \ref{fig:tsne_output}. This change supports the conclusion of Lemma \ref{thm:time_dec}. Moreover, we notice the center of Fig. \ref{fig:tsne_output} clusters some points in light red, which is likely attributable to the low perturbation at night.

\begin{figure*}[p]
    \subfloat[]{
    \includegraphics[width=0.5\linewidth, trim={15 15 15 15}, clip]{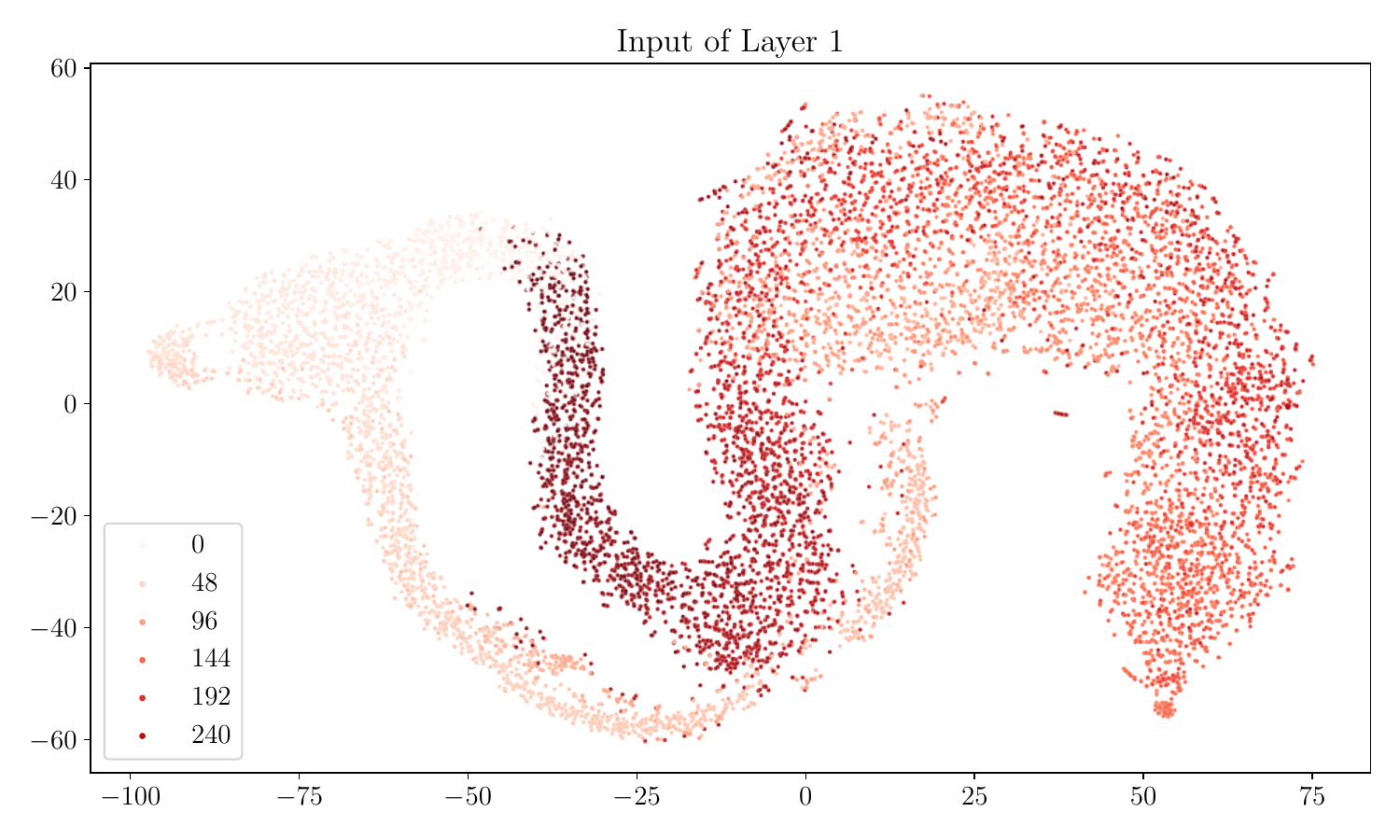}
    \label{fig:tsne_input}
    }
    \hfill
    \subfloat[]{
    \includegraphics[width=0.5\linewidth, trim={15 15 15 15}, clip]{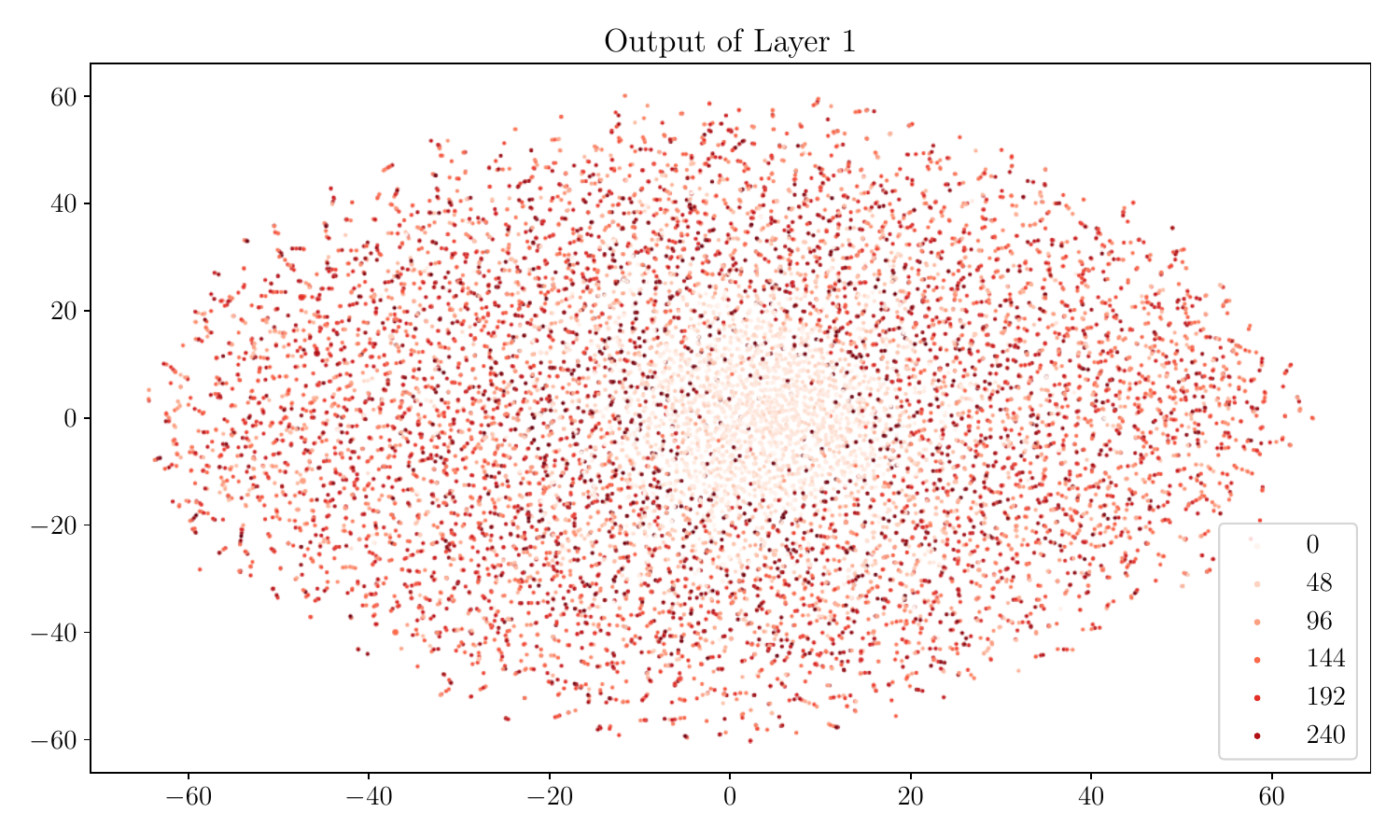}
    \label{fig:tsne_output}
    }
    \caption{Visualization of input time series and output residual signals of the first layer by t-SNE. The visualization of the input shows a distinct distribution and that disappears in the visualization of the output. (a) Visualization of the input time series. (b) Visualization of the output residual signals.}
    \label{fig:vis_tsne}
\end{figure*}

\subsection{Interpretability}\label{sec:interp}

In this section, we provide a case study of interpretability by visualizing and summarizing the contributions of the entries in the memory bank. The contributions to the time step 1339 with the periodic time step of 72 are plotted in Fig. \ref{fig:interp_72_100}. We also plot the traffic flow in the training set for reference. 
Fig. \ref{fig:interp_72_100} visualizes the contribution for time step 1339, where the output of the first layer dominates the final prediction (refer to Fig. \ref{fig:vis_dec_tsnn}). That pattern of the contribution reflects the dominance of the first layer, where more time points with the same $p(i)$ are assigned with high contribution scores, and the time points with different $p(i)$ are assigned with negligible scores. 
Based on the contributions of entries, we can directly observe the instance-wise dependency and infer the dependency on the temporal information. 

To further explore the potential of the contributions of entries, we select 3 days in the test set and summarize the instance-wise contributions to the forms of day-to-day and day-of-week contributions. We select 3 days from sensor 100 in the test sets of the PEMS08 dataset, which includes a Thursday (from time step 1267 to 1554), a Saturday (from time step 2131 to 2418), and a Sunday (from time step 2419 to 2706). 
For each selected day, we first obtain the contributions of each entry in the memory bank to every time step in the selected day. Then we calculate the sum of contributions from each day in the memory bank. The results are plotted on the left side of Fig. \ref{fig:interp_diw}. 
Because of the strong periodicity in the traffic data, the curves of contributions form regular patterns that reflect a larger periodicity, which repeats every seven days.
We further extract the day-of-week labels of the entries and summarize contributions of each day of week to the selected days in the test set, where the results are shown on the right side of Fig. \ref{fig:interp_diw}. We can observe that, except for the top one for the selected Thursday, the highest contributions for the Saturday and Sunday accurately reflect the correct day-of-week labels. Although the highest contribution for the Thursday does not match the correct label, it still reflects the property of a weekday instead of a weekend. Note that the day-of-week labels are not considered during the prediction process. This experiment shows that the contributions can be traced to interpret the prediction, demonstrating the interpretability of TSNN and its potential to provide more insightful information to users.

\begin{figure*}
    \centering
    \includegraphics[width=1\linewidth]{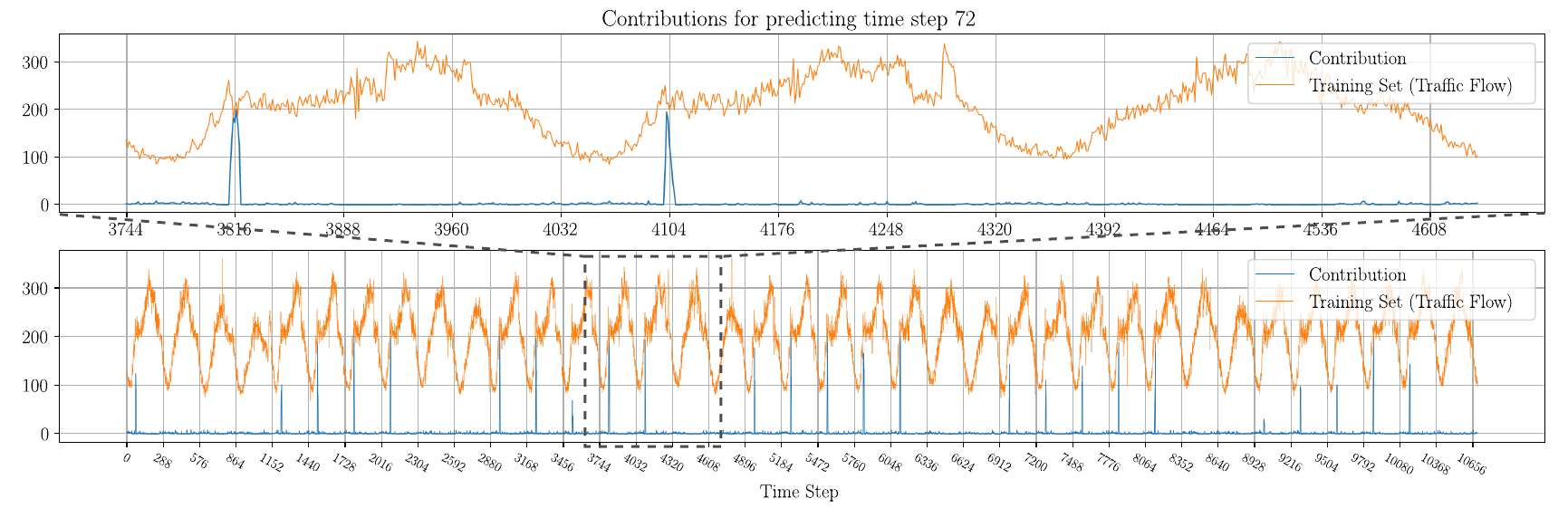}
    \caption{Visualizations of the contributions to the time step 1339 (with the periodic time step of 72) of sensor 100 in the test set of PEMS08.}
    \label{fig:interp_72_100}
\end{figure*}

\begin{figure*}
    \centering
    \includegraphics[width=0.95\linewidth, trim={0 0 0 0}, clip]{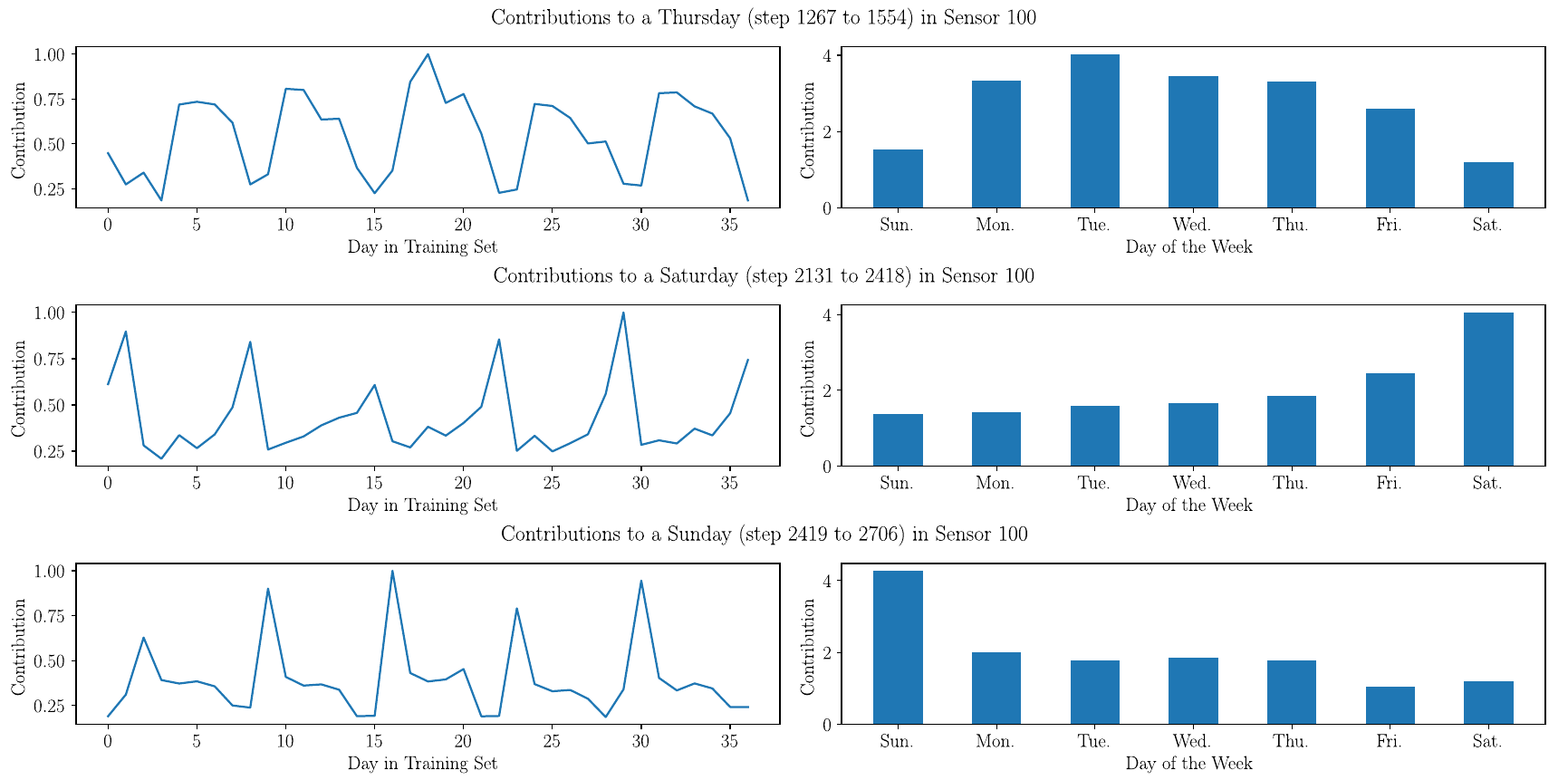}
    \caption{Visualizations of the day-to-day contributions and the day of week contributions to a Thursday, a Saturday, and a Sunday of sensor 100 in the test set of PEMS08.}
    \label{fig:interp_diw}
\end{figure*}




\subsection{Efficiency Study} \label{sec:effi_stu}

In this section, we investigate the efficiency of TSNN in two aspects: memory usage and processing time. We investigate the efficiency of TSNN using two implementation strategies:
\begin{enumerate}
    \item Standard Strategy (denoted as TSNN): This strategy separates the memory bank construction and prediction stages, storing the memory banks of all layers to meet practical requirements for real-world applications.
    \item Memory-Efficient Strategy (denoted as TSNN$^\dagger$): This strategy optimizes space complexity by retaining only the memory banks of the current and preceding layers during the layer prediction, thereby significantly reducing memory usage.
\end{enumerate}
We present the results on ETTh1, PEMS08, and PEMS07 to represent the small, medium, and large-scale datasets, respectively. We report the processing time of TSNN as the sum of ``memory bank construction time + prediction time''. Note that in real-world applications, the memory bank construction phase conceptually corresponds to the training stage of neural networks, as both are prerequisite steps before deployment.
For baselines, we report the total runtime (including training, validation, and testing) after training for 100 epochs using the default configurations in BasicTS. The results are shown in Table \ref{tab:mem_speed}.

For the memory usage, the linear baseline DLinear demonstrates the lowest memory consumption due to its simple structure. Deep learning models like DCRNN and GWNet exhibit moderate to high memory usage depending on the graph size. For the standard TSNN, the memory usage exhibits a linear growth with the number of layers. On the large-scale PEMS07 dataset, the memory consumption of a 10-layer TSNN reaches 19.8 GB, exceeding that of DCRNN. This indicates that storing full historical residuals for all layers can be a bottleneck for large datasets. In contrast, the proposed memory-efficient strategy (TSNN$^\dagger$) effectively resolves this bottleneck. The memory usage of TSNN$^\dagger$ remains nearly constant regardless of the number of layers. On PEMS07, TSNN$^\dagger$ reduces the memory usage of 10 layers from 19836 MB to 6132 MB, making it comparable to GWNet and more efficient than DCRNN. 

For computational efficiency, TSNN generally consumes much less time to construct the memory bank compared with the training time of the baselines. Especially in the small datasets (ETTh1), TSNN only needs a few seconds to produce prediction results from the training data. That shows the potential of real-time application of TSNN on small-scale data. This characteristic is particularly advantageous for cold-start scenarios. However, when the scale of data increases, the disadvantage of the relatively long prediction time becomes apparent, which limits the applicability of TSNN. More discussion on space and time complexity and potential improvement is presented in Section \ref{sec:space_comp}. 

\begin{table}[htbp]
\centering
\caption{Memory usages (in megabytes) and runtime/processing time (in seconds) on datasets ETTh1, PEMS08, and PEMS07.}
\label{tab:mem_speed}

\scalebox{0.72}{
\small
\begin{tabular}{c|c|c|c|c|c|c}

\toprule
\midrule

\textbf{Datasets} & \multicolumn{2}{c|}{ETTh1} & \multicolumn{2}{c|}{PEMS08} & \multicolumn{2}{c}{PEMS07} \tabularnewline

\midrule
\midrule

\textbf{Methods} & \textbf{Memory} & \textbf{Time} &\textbf{Memory} & \textbf{Time} & \textbf{Memory} & \textbf{Time} \tabularnewline

\midrule

DCRNN & - & - & 3256 & 2728 & 18272 & 19563\tabularnewline
\midrule
GWNet & - & - & 862 & 1419 & 4480 & 8893 \tabularnewline
\midrule
DLinear & 56 & 50 & 241 & 89 & 1845 & 305 \tabularnewline
\midrule
Informer & 336 & 290 & 446 & 398 & 2092 & 683 \tabularnewline
\midrule
\midrule
TSNN (1 layer) & 48 & 0+1.1 & 920 & 0+21 & 7520 & 0+112 \tabularnewline
\midrule
TSNN (2 layers) & 70 & 1.3+1.2 & 1092 & 27+25 & 8892 & 138+150 \tabularnewline
\midrule
TSNN (5 layers) & 290 & 2.2+1.3 & 1692 & 63+37 & 12996 & 653+271 \tabularnewline
\midrule
TSNN (10 layers) & 310 & 4.2+2.0 & 2436 & 105+48 & 19836 & 1307+500 \tabularnewline
\midrule
\midrule
TSNN$^\dagger$ (1 layer) & 50 & 1.1 & 689 & 21 & 5698 & 113 \tabularnewline
\midrule
TSNN$^\dagger$ (2 layers) & 252 & 2.4 & 1020 & 53 & 6116 & 298\tabularnewline
\midrule
TSNN$^\dagger$ (5 layers) & 254 & 3.5 & 1024 & 84 & 6122 & 926\tabularnewline
\midrule
TSNN$^\dagger$ (10 layers) & 258 & 5.2 & 1028 & 152 & 6132 & 1872\tabularnewline
\bottomrule

\end{tabular}
}
\end{table}

\section{Discussion} \label{sec:limit}

In this section, we discuss some observed limitations of TSNN and try to provide possible directions for future work.

\subsection{Out-of-distribution Problem} \label{sec:ood_prob}

Since TSNN generates layer predictions by linear combinations of historical signals (Equation (\ref{equ:yl_sum})), it inherently struggles to handle the out-of-distribution (OOD) problem. While subsequent layers iteratively correct these errors (as observed in Fig. \ref{fig:vis_dec_nomean}), the convergence is inefficient for severe OOD instances, suggesting the need for more robust extrapolation mechanisms.
One of the future works can be overcoming the OOD problem.
A possible solution is to extract features by analyzing the patterns inside the time series, such as trend and seasonality. Thereby, the model can expand the distribution range of results through richer feature combinations.
In addition, as we mentioned in Section \ref{sec:main_res}, utilizing spatial information is another potential way to provide extra features to the model and improve the performance of TSNN.

\subsection{Space and Time Complexity} \label{sec:space_comp}

Because the memory bank stores entries for the instances in the training set, the space complexity for a sensor is $\mathcal{O}(\vert \mathcal{T} \vert \cdot TL)$ where $\vert \mathcal{T} \vert$ denotes the number of instances, $T$ and $L$ denote the time steps in prediction and number of layers. 
The space complexity would be even larger when multiple sensors are considered simultaneously. 
Fortunately, the memory bank and prediction process can be implemented in a modular manner. Based on the target sensors and the accuracy requirement, the computation unit can only load the needed part of the memory bank, reducing the space and computation requirements. 

However, the total space complexity and time complexity are still the limitations of the proposed TSNN, which limit its capability to extend to other tasks, such as long-term time series forecasting. 
A smaller memory bank occupies less memory and leads to less computation.
Therefore, future works can aim to reduce the memory bank size while minimizing accuracy loss. 
Some techniques might be useful to reduce the complexity, such as distilling \cite{hinton2015distilling} and Mixture-of-Experts \cite{cai2024survey}. 

\section{Conclusion} \label{sec:conclu}
This paper presented TSNN, a novel non-parametric framework for traffic flow forecasting. This framework leverages memory banks to store the features of the time series and decoupled signals of each instance in layers, enabling inherent interpretability to TSNN. Two proposed decoupling methods, time-wise decoupling and mean decoupling, enhance the performance of the framework. Through extensive experiments, TSNN shows competitive performance compared to the typical deep learning methods. The experiments also verify the effectiveness of the proposed components and visualize the interpretability. Despite TSNN's promising performance, the limitations, such as the out-of-distribution problem and high complexity, exist and remain areas for further exploration. 
Future work can aim to address these challenges.
By addressing these limitations, TSNN can be a more robust and efficient solution for time series forecasting.



\bibliographystyle{ieeetr}
\bibliography{refs.bib}

\newpage

\vfill

\end{document}